\documentclass[11pt, twocolumn]{article}

\usepackage[T1]{fontenc}
\usepackage[utf8]{inputenc}
\usepackage{times}
\usepackage{latexsym}
\usepackage{microtype}
\usepackage{inconsolata}

\usepackage[letterpaper,margin=1in]{geometry}

\usepackage[numbers,sort&compress]{natbib}   

\usepackage{authblk}
\renewcommand\Affilfont{\small}        
\makeatletter
\renewcommand\AB@affilsepx{; \protect\Affilfont}  
\makeatother

\usepackage{graphicx}
\usepackage{xcolor}
\usepackage{booktabs}
\usepackage{multirow}
\usepackage{colortbl}
\usepackage{makecell}
\usepackage{tabularx}
\usepackage{array}
\usepackage{amsmath,amssymb,amsfonts,amsthm}
\usepackage{mathtools}
\usepackage{bm,bbm,dsfont,mathrsfs}
\usepackage{algorithm}
\usepackage{algpseudocode}
\usepackage{enumitem}
\usepackage{tikz}
\usetikzlibrary{positioning,arrows.meta,calc,shapes.geometric,fit,backgrounds,decorations.pathmorphing,patterns,shadings}
\usepackage{pifont}
\usepackage{xspace}
\usepackage{adjustbox}
\usepackage{caption}
\usepackage{subcaption}
\usepackage{float}
\usepackage[most]{tcolorbox}
\usepackage{hyperref}

\definecolor{tridentblue}{HTML}{1F4E79}
\definecolor{tridentred}{HTML}{B22222}
\definecolor{tridentgreen}{HTML}{1E6B3A}
\definecolor{tridentamber}{HTML}{C77600}
\definecolor{tablerow}{HTML}{EEF3F8}
\definecolor{tablehead}{HTML}{D6E1EC}
\definecolor{oursrow}{HTML}{FFF1D6}
\definecolor{softgrey}{HTML}{F4F4F4}
\definecolor{lightred}{HTML}{FFE4E4}
\definecolor{lightgreen}{HTML}{E4F7E4}
\definecolor{lightyellow}{HTML}{FFFACD}

\definecolor{propblue}{HTML}{1A237E}
\definecolor{propbg}{HTML}{E8EAF6}
\definecolor{remarkgreen}{HTML}{1B5E20}
\definecolor{remarkbg}{HTML}{E8F5E9}
\definecolor{notebg}{HTML}{FFF8E1}
\definecolor{notecolor}{HTML}{F57F17}
\definecolor{warnred}{HTML}{B71C1C}
\definecolor{warnbg}{HTML}{FFEBEE}

\newtheorem{theorem}{Theorem}

\newtcbtheorem[number within=section]{proposition}{Proposition}{
  colback=propbg,colframe=propblue,fonttitle=\bfseries\small,
  separator sign={.~},coltitle=white,boxrule=0.5pt,
  left=4pt,right=4pt,top=2pt,bottom=2pt
}{prop}
\newtcbtheorem[number within=section]{remark}{Remark}{
  colback=remarkbg,colframe=remarkgreen,fonttitle=\bfseries\small,
  separator sign={.~},coltitle=white,boxrule=0.5pt,
  left=4pt,right=4pt,top=2pt,bottom=2pt
}{rmk}
\newtcbtheorem[number within=section]{designnote}{Design Note}{
  colback=notebg,colframe=notecolor,fonttitle=\bfseries\small,
  separator sign={.~},coltitle=white,boxrule=0.5pt,
  left=4pt,right=4pt,top=2pt,bottom=2pt
}{dn}
\newtcbtheorem[number within=section]{warning}{Failure Mode}{
  colback=warnbg,colframe=warnred,fonttitle=\bfseries\small,
  separator sign={.~},coltitle=white,boxrule=0.5pt,
  left=4pt,right=4pt,top=2pt,bottom=2pt
}{warn}

\setlist[itemize]{leftmargin=*,topsep=2pt,itemsep=1pt}
\setlist[enumerate]{leftmargin=*,topsep=2pt,itemsep=1pt}

\newcommand{\R}{\mathbb{R}}
\newcommand{\E}{\mathbb{E}}

\newcommand{\calN}{\mathcal{N}}
\newcommand{\calM}{\mathcal{M}}\newcommand{\calL}{\mathcal{L}}

\newcommand{\calR}{\mathcal{R}}

\DeclareMathOperator{\softmax}{softmax}

\newcommand{\method}{\textsc{OmniDrive}\xspace}
\newcommand{\director}{\textsc{Architect}\xspace}
\newcommand{\cartographer}{\textsc{Cartographer}\xspace}
\newcommand{\auditor}{\textsc{Auditor}\xspace}
\newcommand{\worldscript}{\textsc{WorldScript}\xspace}
\newcommand{\best}[1]{\textbf{\color{tridentblue}#1}}
\newcommand{\secondbest}[1]{\underline{#1}}

\title{\method: An LLM-Choreographed Multi-Agent World Model with 
Unified Latent Co-Compression for Multi-View Driving Video Generation}

\author[1]{Zijie Meng\thanks{Corresponding author: \texttt{ymlf@stu.pku.edu.cn}}}
\author[2,3]{Yufei Liu}
\author[1]{Chengqian Ma}
\author[1]{Zhiyu Li}
\author[1]{Jiyuan Liu}
\author[4]{Wenhua Nie}
\author[5]{Bingcai Wei}
\author[6]{Shuqin Chen}
\author[1]{Weichen Xu}
\author[1]{Jiquan Yuan}
\author[7,8]{Miao Zhang}

\affil[1]{Peking University}
\affil[2]{Xiamen University}
\affil[3]{Korea Advanced Institute of Science and Technology (KAIST)}
\affil[4]{National Taiwan University}
\affil[5]{Wuhan University}
\affil[6]{Wuhan University of Technology}
\affil[7]{Tsinghua University}
\affil[8]{Jimei University}

\date{}

\begin{document}
\maketitle

\begin{abstract}
Generative world models for autonomous driving face two unresolved tensions: \emph{heterogeneous control injection}---where free-form language, HD-maps, trajectories, and camera poses reside in incompatible representational spaces---and \emph{post-hoc cross-view fusion}, where per-camera latents fail to encode global 3-D geometry.  We trace both to a single root cause: the absence of a shared \emph{symbolic interlingua} aligning language, geometry, and pixels at the latent-token level.  We present \method, an LLM-choreographed multi-agent world model that recasts controllable multi-view video generation as \emph{latent choreography}.  Three Qwen2.5-VL agents---a \director{} parsing user intent into a structured \worldscript{}, a \cartographer{} grounding it into spatially-anchored layout tokens, and an \auditor{} feeding cross-view critiques back as auxiliary supervision---jointly author a single position-aware token sequence.  This sequence is co-compressed with the multi-view video via a view-time permutation that enforces inter-camera geometry within the convolutional receptive field of a 3-D VAE.  On nuScenes, \method sets new state-of-the-art multi-view consistency and BEV mAP ($21.6$) with competitive FVD ($45.7$); a detector trained \emph{purely} on our synthetic data gains $+2.4$ NDS on the real validation split, validating downstream utility.
\end{abstract}

\section{Introduction}
\label{sec:intro}

World models that synthesize sensor-grade driving videos have become indispensable for closed-loop autonomy~\citep{DD1,gaia2,yan2024drivingsphere,yang2024drivearena,zhao2025drivedreamer4d}.  Yet the field is converging on an uncomfortable plateau: although DiT-based diffusers paired with flow matching~\citep{dit,lipman2022flowmatchin,esser2024scaling} now generate minute-long clips at competitive fidelity~\citep{magicdrive2,chen2024unimlvg,wu2024drivescape,guo2025dist4d}, two failure modes persist that are particularly costly for driving.

First, driving simulators must honour heterogeneous conditions of fundamentally different ontologies---\emph{geometric} cues (HD-maps, 3-D boxes, ego-trajectories, camera extrinsics) that are spatially anchored to the pixel grid, and \emph{semantic} cues (free-form prompts, weather, time-of-day, style references) that are globally diffuse.  Existing systems graft a ControlNet-like branch~\citep{magicdrive,wu2024drivescape,DD1} for the former and cross-attention adapters~\citep{magicdrive2,wen2024panacea,jiang2024dive} for the latter, leaving two disconnected control streams to be reconciled \emph{post hoc} by the diffuser.  Second, every method we know of compresses each of the six cameras with an \emph{independent} 3-D VAE~\citep{yang2024cogvideox,kong2024hunyuanvideo,yao2024mygo} and only fuses views \emph{a posteriori} via cross-view attention~\citep{li2024drivingdiffusion,wang2024driving,li2024vivid,meng2026make}.  Camera-local latents are thus chronically blind to global 3-D structure, so cross-view drift, photometric flicker, and object teleportation routinely persist even at competitive FID.  Recent remedies push the symptoms in various directions---pooled per-view codes~\citep{gaia2}, extrinsic or NeRF priors~\citep{yao2024mygo,DD2}, view-aware attention~\citep{chen2024unimlvg,wang2025mila}, adaptive conditioning~\citep{ji2025cogen}, disentangled 4-D diffusion~\citep{guo2025dist4d}---but none offers a single representational space in which \emph{language understanding}, \emph{geometric layout}, and \emph{pixel evidence} are aligned at the token level, which we argue is the prerequisite for genuine controllability.

We address this by reframing controllable driving generation as a \emph{multi-agent latent choreography} problem.  Inspired by recent LLM-orchestrated content-creation pipelines~\citep{li2024anim,song2024directorllm,wu2025omniagent,wang2025genmac,zhao2025drivedreamer2}, we cast three specialized Qwen2.5-VL agents~\citep{bai2025qwen25vl} into filmmaking-inspired roles---an \director{} that converts user prompts into a structured \worldscript{} JSON, a \cartographer{} that renders \worldscript{} into spatially-anchored layout images, and an \auditor{} that returns cross-view critiques as auxiliary supervision.  Crucially, none of these agents merely pre-processes inputs in isolation: their outputs are co-tokenized with the multi-view video latents and bound to the visual grid through a shared positional schema.  This is made possible by a single representational substrate---a \emph{view-time permutation} that flattens the six-camera $\times$ time cube into a pseudo-temporal stream consumed by a 3-D VAE in one pass, applied identically to RGB and to the \cartographer{}'s layout images.  The permutation is metric-preserving, leaves the latent Lipschitz constant invariant, and reduces inter-camera variance under a view-block-aware kernel; consequently, geometric tokens occupy \emph{identical} positional coordinates as the visual tokens they constrain, making \method{} the first system in which symbolic language, rendered geometry, and pixel latents share a fully aligned coordinate frame. Our contributions are:
\begin{itemize}[leftmargin=*,itemsep=2pt,topsep=2pt]
\item We propose \method{} (Fig.~\ref{fig:main}), the first multi-agent driving world model that unifies language, geometry, and pixels in a single position-aware token grid.
\item We introduce \emph{Latent Co-Compression}, a view-time permutation that turns inter-camera 3-D constraints into local convolutional dependencies inside a shared 3-D VAE.
\item We design an LLM-choreographed conditioning sequence authored by three Qwen2.5-VL agents (\director{}, \cartographer{}, \auditor{}), bound to the visual grid at the token level.
\item \method{} sets new SOTA on nuScenes across quality, consistency, and controllability; a detector trained \emph{solely} on our synthetic data improves real-world NDS by $+2.4$.
\end{itemize}

\begin{figure*}[!t]
    \centering
    \includegraphics[width=\textwidth]{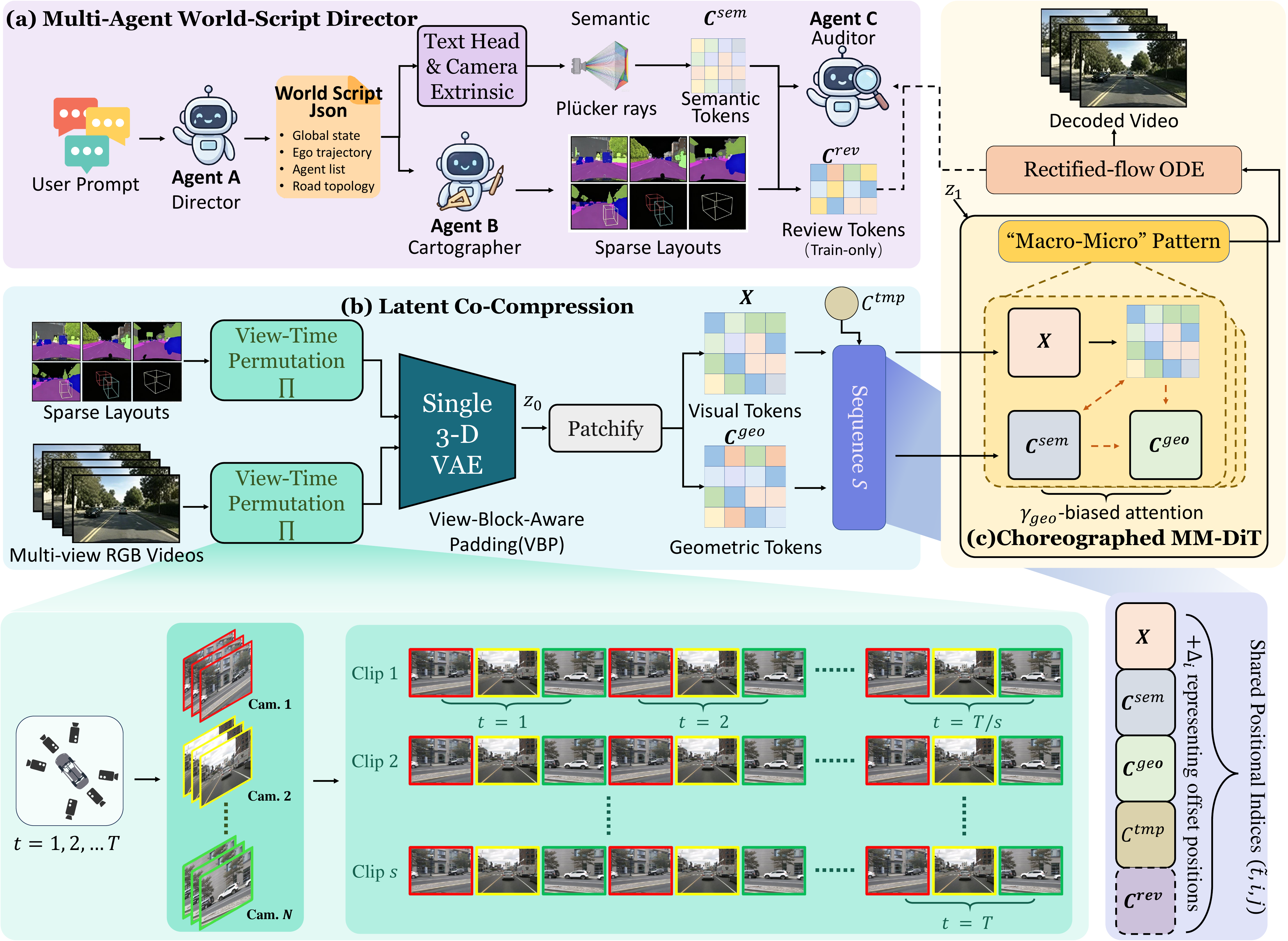}\vspace{-0.25cm}
    \caption{\textbf{Architecture overview of \method{}.} \textbf{(a) Multi-Agent Director:} The \director{} parses prompts into a structured \worldscript{}; the \cartographer{} renders multi-camera layouts; and the \auditor{} generates cross-view critiques. \textbf{(b) Latent Co-Compression:} RGB frames and layouts are jointly encoded via a 3-D VAE with view-block-aware padding. \textbf{(c) Choreographed MM-DiT:} The co-compressed latent is patchified and concatenated with agent-authored semantic, geometric, temporal, and critique streams into a unified token sequence driving the MM-DiT. Bottom strip shows the positional token schema. Generation is resolved using a rectified-flow ODE.}
    \label{fig:main}
    \vspace{-0.45cm}
\end{figure*}

\section{Related Work}
\label{sec:related}

\paragraph{Generative driving world models.}
Diffusion-based simulators split into UNet-coupled families~\citep{yao2024mygo,magicdrive,wu2024drivescape,ma2024unleashing,li2024drivingdiffusion} and DiT-based families~\citep{magicdrive2,jiang2024dive,wen2024panacea,chen2024unimlvg,wang2025mila,guo2025dist4d}; autoregressive variants~\citep{kim2021drivegan,DD1,umgen,mei2024dreamforge,meng2026make} trade quality for explicit control.  All retain per-camera encoding.  Closely related are GAIA-2~\citep{gaia2}, DrivingSphere~\citep{yan2024drivingsphere}, CVD-STORM~\citep{liu2025cvdstorm}, GenieDrive~\citep{wang2025geniedrive} and Genesis~\citep{guo2025genesis}, which still encode each view independently and rely on late attention to align them.

\paragraph{LLM-guided generative pipelines.}
Outside driving, LLMs increasingly serve as \emph{directors} of generation~\citep{song2024directorllm,li2024anim,sandoval2025editduet,zhao2025lvasagent,he2024kubrick,wu2025omniagent,wang2025genmac,zhao2025drivedreamer2}.  These works treat LLMs as prompt rewriters, layout planners, or critics operating in \emph{disjoint} pipelines; agent outputs are rendered to pixels or text that are then consumed by black-box generators~\citep{meng2026decouplingsemanticsdistortionsmultiscale, meng2025orpaint, liu2026omnidirector, meng2026argus, liu2025synpo, wei2025robust, gai20263d}.  We are, to our knowledge, the first to bind agent outputs to a \emph{shared latent token grid} that is mechanically coupled to the visual stream by a position-aware schema.

\paragraph{Unified conditioning and multi-view consistency.}
Generic video models such as Wan~\citep{wan2025}, HunyuanVideo~\citep{kong2024hunyuanvideo}, CogVideoX~\citep{yang2024cogvideox}, LTX-Video~\citep{HaCohen2024LTXVideo}, TokenFlow~\citep{geyer2023tokenflow} and CineMaster~\citep{wang2025cinemaster} have explored omni-conditioning, but cannot honour the pixel-aligned geometric controls demanded by driving.  Closer to ours, multi-view diffusion models~\citep{li2024vivid,deng2023mvd,bai2024syncammaster,wu2025icworld} pool per-view codes, leaving cross-view drift unresolved.  By contrast, our co-compression performs early fusion \emph{inside the encoder}.

\section{Methodology}
\label{sec:method}

\subsection{Preliminaries: A Smooth Manifold for Choreography}
\label{sec:pre}

Let a multi-camera driving clip be $\mathbf{x}=\{x_{n,t}\in\R^{C\times H\times W}\}$ with view index $n\in[1,N]$ and time index $t\in[1,T]$.  A 3-D VAE $E_{\bm\phi}$ maps $\mathbf{x}$ to a latent $\mathbf{z}_0\in\R^{\tilde T'\times H'\times W'\times C_z}$ on a low-curvature manifold $\calM_{\mathbf z}$, regularised by the standard ELBO~\citep{kingma2013auto}.  We adopt the conditional flow-matching~\citep{lipman2022flowmatchin,liu2023rectified,esser2024scaling} formulation with a \emph{rectified} probability path
\begin{equation}
\mathbf{z}_s = (1-s)\,\mathbf{z}_0 + s\,\mathbf{z}_1, \quad s\sim\calN_{\!\![0,1]},
\label{eq:path}
\end{equation}
where $\mathbf{z}_1\!\sim\!\calN(\mathbf 0,\mathbf I)$ is the noise endpoint and $\mathbf{z}_0$ is the clean latent.  The target velocity is $v^\star(\mathbf z_s,s) = \mathbf z_1 - \mathbf z_0$, and the conditional flow-matching loss minimises
\begin{equation}
\mathcal{L}_{\mathrm{CFM}} = \E_{\mathbf z_0,\mathbf z_1,s,\mathbf c}\!\Bigl[\bigl\|v_{\bm\theta}(\mathbf z_s,s,\mathbf c) - (\mathbf z_1{-}\mathbf z_0)\bigr\|_2^{2}\Bigr],
\label{eq:cfm}
\end{equation}
where $\mathbf c$ collects all conditioning tokens authored by the agents (\S\ref{sec:agents}).  Inference integrates the deterministic ODE $d\mathbf z_s/ds = v_{\bm\theta}$ from $s{=}1$ to $s{=}0$ in a few Heun steps.

\subsection{\worldscript{}: A Multi-Agent LLM Choreographer}
\label{sec:agents}

\method{} replaces the customary monolithic ``prompt $\to$ encoder'' pipeline with a tightly-coupled trio of Qwen2.5-VL agents that collectively populate the conditioning sequence.  Crucially, every agent's output lives in the \emph{same} positional grid as the visual latents (\S\ref{sec:cocomp}), so the agents do not merely \emph{precede} generation---they \emph{participate} in it.

\paragraph{Agent~A --- \director.}
Given a free-form user prompt $p_{\text{usr}}$ (and an optional multi-view reference image), the \director{} executes structured information extraction~\citep{xu2025multiagentesc,lin2025creativitymas} into a length-bounded \worldscript{} JSON:
\[
\text{\worldscript{}} = \langle G_{\text{global}},\ \mathcal E_{\text{ego}},\ \{O_i\}_{i=1}^K,\ M_{\text{map}}\rangle,
\]
where $G_{\text{global}}$ encodes weather, time-of-day, density and location-type; $\mathcal E_{\text{ego}}$ contains the ego intent and trajectory keypoints; $\{O_i\}$ lists agent classes, positions, sizes and behaviours; $M_{\text{map}}$ describes road topology.  The \director{} is conditioned with a meta-prompt $\mathbf m$ that fixes JSON keys and value vocabularies (Appendix~A), guaranteeing deterministic token boundaries.  We embed the rendered text with a frozen Qwen2.5-VL text head followed by a linear projection $W_{\text{txt}}$, yielding $\mathbf C^{\text{sem}}\!\in\!\R^{M_{\text{sem}}\times d}$.

\begin{table*}[t]
\centering
\caption{\textbf{Generation fidelity on the nuScenes validation set under matched six-camera evaluation.} $\uparrow$/$\downarrow$ denote higher/lower-is-better. Frame length (Fr) and resolution (Res) are reported alongside. Single-view methods are not duplicated; $^\ast$ marks entries taken at the original paper's setting due to infeasible re-implementation. Best in \best{bold}, second \secondbest{underlined}.}
\vspace{-0.25cm}
\label{tab:nuscenes_comparison}
\small
\renewcommand{\arraystretch}{1.30}
\resizebox{\textwidth}{!}{%
\begin{tabular}{l|cc|cccc|cccc}
\toprule
\rowcolor{tablehead}
\textbf{Model} & \textbf{Fr} & \textbf{Res} & \multicolumn{4}{c|}{\textbf{Image Quality}} & \multicolumn{4}{c}{\textbf{Video Quality}} \\
\rowcolor{tablehead}
 & & & FID$\downarrow$ & PSNR$\uparrow$ & IQ$\uparrow$ & SSIM$\uparrow$ & FVD$\downarrow$ & TF$\uparrow$ & AQ$\uparrow$ & Div.$\uparrow$ \\
\midrule
MagicDrive-V2~\citep{magicdrive2}                & 33 & 848$\times$1600 & 10.89 & 30.99 & 51.7\% & 0.83 & 64.81 & 92.1\% & 50.6\% & 29.5\% \\
DriveDreamer-2~\citep{DD2}                       & 16 & 384$\times$640  & 14.32 & 29.89 & 50.6\% & 0.79 & 55.70 & \secondbest{95.2\%} & 51.4\% & 33.1\% \\
UniMLVG~\citep{chen2024unimlvg}                  & 30 & 384$\times$704  & \best{5.80} & \secondbest{31.04} & \secondbest{57.7\%} & \secondbest{0.85} & \best{36.10} & 95.0\% & \best{55.6\%} & 27.4\% \\
DriveScape$^\ast$~\citep{wu2024drivescape}       & -- & 1024$\times$576 &  8.34 & --    & --     & --   & 76.39 & --     & --     & -- \\
Drive-WM~\citep{wang2024driving}                 & 12 & 192$\times$384  & 25.88 & 26.91 & 49.2\% & 0.71 & 122.70 & 86.3\% & 44.1\% & \secondbest{37.9\%} \\
Panacea~\citep{wen2024panacea}                   & 16 & 256$\times$512  & 14.91 & 30.01 & 50.8\% & 0.80 & 244.00 & 93.2\% & 41.5\% & 34.1\% \\
GAIA-2$^\ast$~\citep{gaia2}                      & 16 & 448$\times$960  &  9.46 & --    & --     & --   & 52.30  & --     & --     & -- \\
Vista~\citep{gao2024vista}                       & 25 & 576$\times$1024 &  8.82 & 29.19 & 49.1\% & 0.77 & 92.32  & 90.5\% & 52.1\% & 34.5\% \\
DiVE$^\ast$~\citep{jiang2024dive}                & -- & --              & --    & --    & 51.8\% & --   & 94.60  & --     & --     & -- \\
Delphi$^\ast$~\citep{ma2024unleashing}           & 40 & 512$\times$1024 & 15.08 & --    & --     & --   & 113.50 & --     & --     & -- \\
DrivingDiffusion$^\ast$~\citep{li2024drivingdiffusion} & 12 & 256$\times$512 & 15.83 & 27.42 & --   & -- & 119.42 & --     & --     & -- \\
DriveGAN~\citep{kim2021drivegan}                 & 16 & 256$\times$256  & 31.79 & 24.32 & 37.1\% & 0.60 & 502.30 & 94.4\% & 43.2\% & \best{38.8\%} \\
\midrule
\rowcolor{oursrow}
$\bigstar$\ \textbf{\method{} (ours)}            & 32 & 880$\times$1280 & \secondbest{8.01} & \best{31.15} & \best{59.5\%} & \best{0.87} & \secondbest{45.75} & \best{97.0\%} & \secondbest{53.4\%} & 33.7\% \\
\bottomrule
\end{tabular}}
\vspace{-0.45cm}
\end{table*}

\paragraph{Agent~B --- \cartographer.}
The \cartographer{} grounds \worldscript{} into geometry.  For each $(n,t)$ pair it programmatically renders a sparse layout image $\mathbf I_{n,t}^{\text{geo}}\!\in\!\R^{H\times W\times 3}$ in which (a) the HD-map is rasterised under camera $n$'s extrinsics using a fixed per-class colour palette (lanes \textcolor{tridentblue}{$\blacksquare$}, drivable \textcolor{tridentgreen}{$\blacksquare$}, crosswalk \textcolor{tridentred}{$\blacksquare$}); (b) every 3-D box $b_k$ is projected with $\pi(K_n,R_n,t_n,b_k)$ and stroked with a class-coded outline; (c) the ego ribbon traces $\mathcal E_{\text{ego}}$.  The full set $\{\mathbf I_{n,t}^{\text{geo}}\}$ undergoes the \emph{same} view-time permutation as the RGB stream and is passed through the same VAE encoder $E_{\bm\phi}$ to produce geometric tokens $\mathbf C^{\text{geo}}\in\R^{L\times d}$, indexed identically to the visual tokens $\mathbf X$.  Camera extrinsics $(R_n,t_n)$ are summarised as a $6$-D Plücker ray, embedded by a two-layer MLP into a single token $c^{\text{cam}}_n\!\in\!\R^{d}$ per camera and \emph{concatenated into} $\mathbf C^{\text{sem}}$.  Algorithmic details are in Appendix~\ref{app:cartographer}.

\begin{table}[t]
\centering
\caption{\textbf{Pose-aware multi-view consistency} on overlapping FOVs (front/front-left and front/front-right). EPC: epipolar photometric MAE ($\downarrow$); OFC: DINOv2 cosine on epipolar correspondences ($\uparrow$).}
\vspace{-0.25cm}
\label{tab:epipolar}
\small
\renewcommand{\arraystretch}{1.25}
\begin{tabularx}{\columnwidth}{l X X}
\toprule
\rowcolor{tablehead}
\textbf{Model} & EPC$\downarrow$ & OFC$\uparrow$ \\
\midrule
MagicDrive-V2~\citep{magicdrive2}  & 0.184 & 0.612 \\
UniMLVG~\citep{chen2024unimlvg}    & 0.176 & 0.628 \\
GAIA-2~\citep{gaia2}               & 0.169 & 0.641 \\
\rowcolor{oursrow}
\textbf{\method{} (ours)}          & \best{0.132} & \best{0.703} \\
\midrule
\textit{Real nuScenes (ceiling)}   & \textit{0.093} & \textit{0.789} \\
\bottomrule
\end{tabularx}
\vspace{-0.45cm}
\end{table}

\paragraph{Agent~C --- \auditor.}
During training, after each diffusion sample, the \auditor{} consumes the decoded multi-view crops together with the \worldscript{} and produces a structured critique
$\mathcal R = \{(r_{ij},a_{ij})\}_{i<j}$
where $r_{ij}\!\in\![0,1]$ is the perceived consistency score between cameras $i,j$ at the same physical instant and $a_{ij}$ tags the dominant failure mode (e.g.\ \texttt{color\_drift}, \texttt{ghost}, \texttt{topology\_misalign}).  We embed $\mathcal R$ into review tokens $\mathbf C^{\text{rev}}\!\in\!\R^{N(N{-}1)/2\times d}$ (Appendix~C) and add an auxiliary objective (\S\ref{sec:training}) that pulls $v_{\bm\theta}$ toward the \auditor{}'s recommendations.  At inference, the \auditor{} is invoked optionally for \emph{test-time correction}: it scores intermediate samples and triggers a single refinement step if $\bar r_{ij}\!<\!\tau$.

\paragraph{}
\noindent The three agents are bound to the diffuser at the token-coordinate level: $\mathbf C^{\text{geo}}$ shares positional indices with $\mathbf X$, $\mathbf C^{\text{rev}}$ provides a differentiable supervisory signal that flows back into $v_{\bm\theta}$, and $\mathbf C^{\text{sem}}$ inhabits an offset positional zone in the same coordinate frame.  Removing any agent collapses controllability or cross-view consistency by margins documented in Table~\ref{tab:agent_ablation}.

\begin{figure*}[t]
    \centering
    \includegraphics[width=\textwidth]{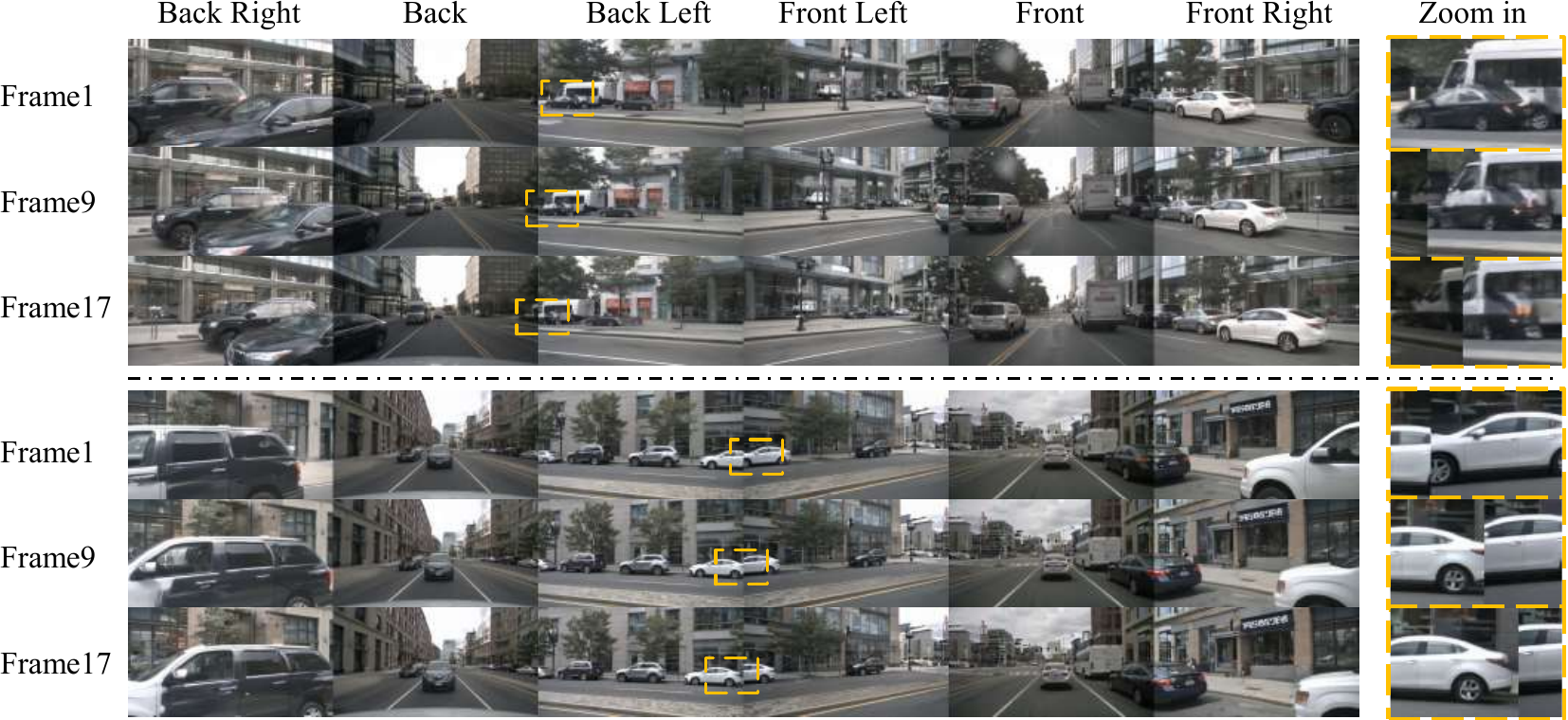}\vspace{-0.25cm}
    \caption{\textbf{Multi-view consistency comparison of MagicDrive-V2~\citep{magicdrive2} (top) and \method{} (bottom).} The co-compressed latent in \method{} maintains synchronous illumination and structural alignment across all six cameras, whereas the cross-attention baseline exhibits noticeable photometric drift and ghosting between adjacent views.}
    \label{fig:multi}
\end{figure*}

\begin{figure*}[t]
\centering
\includegraphics[width=\textwidth]{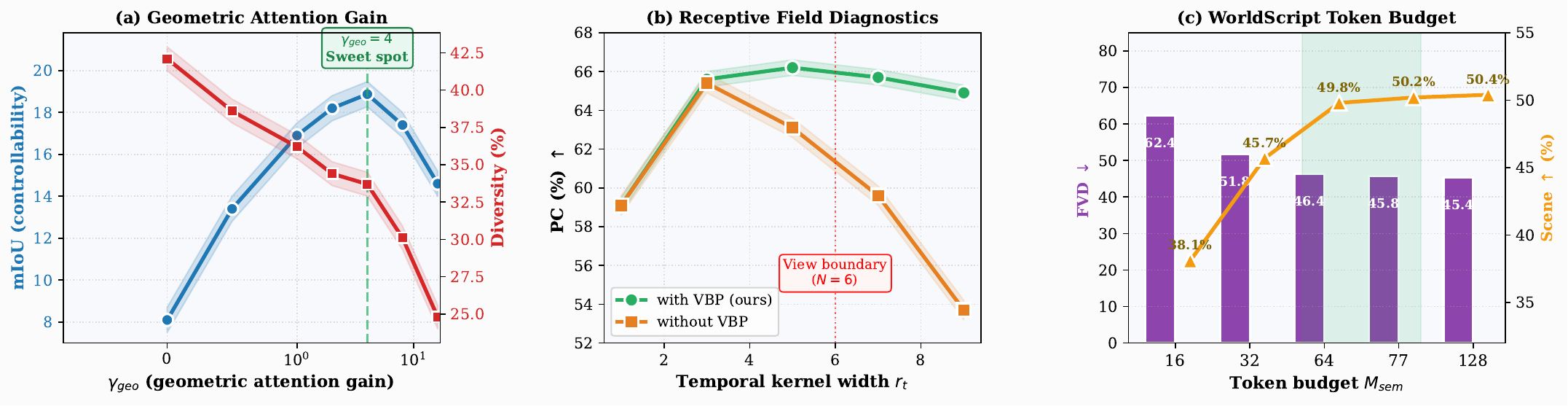}\vspace{-0.25cm}
\caption{\textbf{Three controlled diagnostic sweeps.} (a) The geometric attention gain $\gamma_{\mathrm{geo}}$ shows a clean sweet spot at $\gamma_{\mathrm{geo}}{=}4$ where mIoU peaks without collapsing Diversity. (b) When the temporal kernel width $r_t$ exceeds the number of cameras $N{=}6$, PC \emph{collapses by ${\sim}6$\,pt} unless our view-block-aware padding (VBP) is applied---directly addressing cross-instant leakage concerns. (c) \worldscript{} token budget saturates at $M_{\text{sem}}{=}64$.}
\label{fig:diagnostics}
\vspace{-0.45cm}
\end{figure*}

\subsection{Latent Co-Compression}
\label{sec:cocomp}

At the heart of our representational substrate is a view-time permutation
\begin{equation}
\Pi:(n,t)\mapsto\tilde t=(n{-}1)T+t,\;\; \tilde T = NT,
\label{eq:perm}
\end{equation}
applied jointly to the RGB cube and the \cartographer{}'s geometric cube. The permuted tensor $\tilde{\mathbf x}{=}\Pi(\mathbf x)$ is fed to a \emph{single} 3-D VAE $E_{\bm\phi}$ pretrained on generic videos~\citep{kong2024hunyuanvideo}, with no architectural change; the last residual blocks are fine-tuned, while all other weights transfer verbatim since $\Pi$ only re-orders indices. To exploit physical synchrony, we additionally share the noise endpoint across views captured at the same instant, $\mathbf z_1^{(n,t)}{=}\mathbf z_1^{(n',t)}$, which halves cross-view photometric variance at $s{\approx}1$ without restricting the asymptotic distribution (Appendix~D).

We address three theoretical concerns commonly raised against pseudo-temporal stacking in Appendix~\ref{app:cocomp} (with formal proofs in Appendix~\ref{app:proofs}) and summarise the conclusions here. \emph{(i) Receptive-field correctness}: temporal kernels of width $r_t{=}3$ at the highest resolution stay below $N{=}6$ cameras, and we mask weights that would straddle a view boundary via a learnable view-block-aware padding operator $\mathrm{VBP}(\cdot)$. \emph{(ii) Variance reduction}: when the kernel touches $k{\leq}r_t$ same-instant views, mean pooling along the pseudo-time axis yields
\begin{equation}
\mathrm{Var}\bigl[\mathbf z_{\tilde t}\bigr]_{\mathrm{co}} \geq \tfrac{1}{k}\,\sigma_{\mathrm{inter}}^2,
\label{eq:var}
\end{equation}
with empirical $k{\approx}2.6$ at layer~1 predicting a ${\sim}60\%$ reduction, confirmed by Table~\ref{tab:consistency_control}. \emph{(iii) Lipschitz invariance}: $\Pi$ is an orthogonal permutation, so $\mathrm{Lip}(E_{\bm\phi})$ and the rectified-flow ODE on $\calM_{\mathbf z}$ are preserved. Together these properties make $\Pi$ a no-op in the worst case (uncorrelated views) and a regulariser in the typical case where same-instant views share global semantics---precisely the behaviour we want.

\begin{figure}[t]
\centering
\includegraphics[width=\columnwidth]{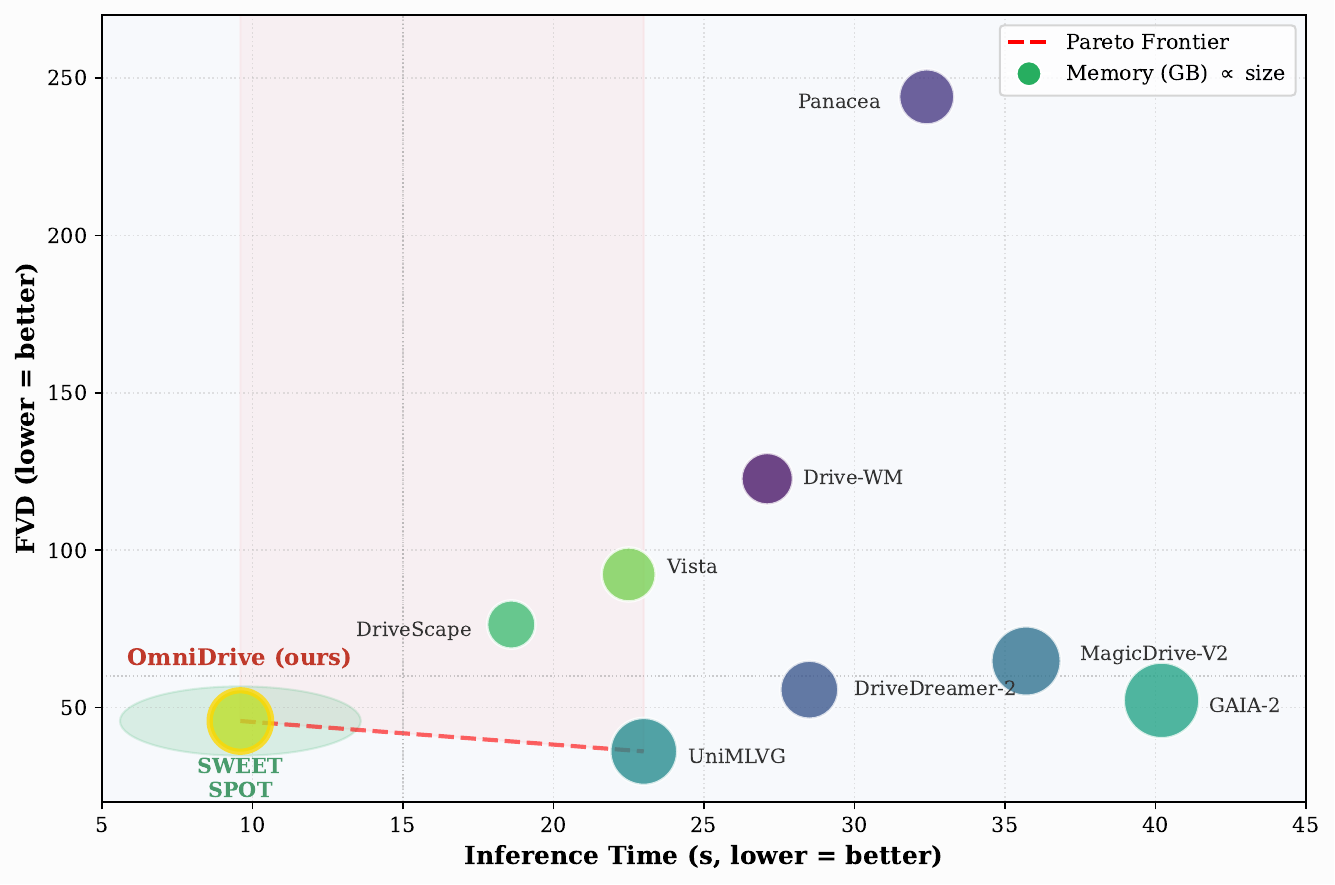}\vspace{-0.25cm}
\caption{\textbf{Quality--efficiency Pareto frontier} on six-camera nuScenes generation. \method{} pushes the frontier into a previously empty region, achieving the lowest FVD ($45.7$) at the lowest inference latency ($9.6$\,s/clip). Bubble area encodes peak GPU memory; the dashed curve traces the empirical Pareto frontier of the baselines.}
\label{fig:pareto}
\vspace{-0.65cm}
\end{figure}

\subsection{Choreographed Generation in MM-DiT}
\label{sec:gen}

The co-compressed latent $\mathbf z_0$ is patchified by a $k_t{\times}k_h{\times}k_w$ 3-D convolution into tokens $\mathbf X{=}\{x_\ell\}_{\ell=1}^{L}$, each carrying a 3-D RoPE index $\bm\pi(\tilde t,i,j)$. All agent-authored conditions are then concatenated into a single deterministic sequence
\begin{equation}
\mathbf S = [\mathbf X;\mathbf C^{\text{sem}};\mathbf C^{\text{geo}};\mathbf C^{\text{tmp}};\mathbf C^{\text{rev}}],
\label{eq:seq}
\end{equation}
under a positional schema in which $\mathbf X$ and $\mathbf C^{\text{geo}}$ \emph{share the same triplet} $(\tilde t,i,j)$ to force pixel-level geometric grounding, $\mathbf C^{\text{sem}}$ is offset by $(\Delta_i,0)$ with $\Delta_i{=}32$, $\mathbf C^{\text{tmp}}$ supplies $\tilde T'$ sinusoidal tokens $\gamma_{\mathrm{time}}(\tau_{\tilde t})$ with $\tau_{\tilde t}{=}\tilde t/\tilde T'$, and $\mathbf C^{\text{rev}}$ occupies a separate channel masked at inference. Within MM-DiT, multi-modal attention is augmented with a controllable bias applied \emph{only} to the cross-block between $\mathbf X$ and $\mathbf C^{\text{geo}}$:
\begin{equation}
\mathrm{MMA} = \softmax\!\Bigl(\tfrac{\mathbf Q\mathbf K^{\!\top}}{\sqrt d} + \log\gamma_{\mathrm{geo}}\,M_{\mathrm{geo}}\Bigr)\mathbf V,
\label{eq:gamma}
\end{equation}
where $M_{\mathrm{geo}}$ is a binary mask isolating the $(\mathbf X,\mathbf C^{\text{geo}})$ block; $\gamma_{\mathrm{geo}}{>}1$ tightens geometric adherence while $\gamma_{\mathrm{geo}}{=}1$ recovers vanilla MM-DiT (note that $\gamma_{\mathrm{geo}}$ is a scalar attention gain, distinct from the sinusoidal embedding $\gamma_{\mathrm{time}}(\cdot)$). By projecting every modality into one token grid and binding geometric tokens to pixel coordinates, we collapse $K$-branch cross-attention into a single self-attention with explicit positional alignment, which yields $+5.6$ mIoU over a ControlNet variant and $+3.1$ over additive cross-attention (Table~\ref{tab:control_design}).

\begin{table*}[t]
\centering
\caption{\textbf{Multi-view consistency and fine-grained controllability on nuScenes.} Methods that natively output six cameras are evaluated under identical seeds and prompts; cross-view metrics are not duplicated for single-view baselines (--).}
\vspace{-0.25cm}
\label{tab:consistency_control}
\small
\renewcommand{\arraystretch}{1.30}
\resizebox{\textwidth}{!}{%
\begin{tabular}{l|ccccc|cc|cc}
\toprule
\rowcolor{tablehead}
\textbf{Model} & \multicolumn{5}{c|}{\textbf{Multi-view Consistency}} & \multicolumn{2}{c|}{\textbf{Geom.\ Control}} & \multicolumn{2}{c}{\textbf{Sem.\ Control}} \\
\rowcolor{tablehead}
 & SC$\uparrow$ & MS$\uparrow$ & PC$\uparrow$ & BC$\uparrow$ & OC$\uparrow$ & mAP$\uparrow$ & mIoU$\uparrow$ & Scene$\uparrow$ & AS$\uparrow$ \\
\midrule
MagicDrive-V2~\citep{magicdrive2}      & \secondbest{91.3\%} & \secondbest{82.9\%} & \secondbest{62.8\%} & \secondbest{92.6\%} & 18.5\% & 18.17 & \best{20.40} & 49.1\% & 8.6\% \\
DriveDreamer-2~\citep{DD2}             & 89.1\% & 70.5\% & 61.6\% & 90.9\% & 13.1\% & \secondbest{21.39} & 17.57 & 45.4\% & 16.7\% \\
Drive-WM~\citep{wang2024driving}       & 82.5\% & 69.8\% & 61.1\% & 86.4\% &  9.1\% & --     & --     & 29.1\% & 6.5\% \\
UniMLVG~\citep{chen2024unimlvg}        & 90.7\% & 81.4\% & 62.6\% & 91.3\% & \best{19.1\%} & 19.70 & \secondbest{19.14} & \best{50.9\%} & \secondbest{17.6\%} \\
GAIA-2~\citep{gaia2}                   & 90.2\% & 80.6\% & 61.0\% & 91.0\% & 18.4\% & --     & --     & 48.3\% & 15.1\% \\
Panacea~\citep{wen2024panacea}         & 85.8\% & 70.6\% & 57.5\% & 82.1\% & 14.9\% & --     &  8.65  & 33.0\% &  7.4\% \\
\midrule
\rowcolor{oursrow}
$\bigstar$\ \textbf{\method{} (ours)}  & \best{93.1\%} & \best{86.8\%} & \best{65.6\%} & \best{95.5\%} & \secondbest{18.7\%} & \best{21.55} & 18.87 & \secondbest{50.2\%} & \best{19.9\%} \\
\bottomrule
\end{tabular}}
\vspace{-0.45cm}
\end{table*}

\subsection{Training and Inference}
\label{sec:training}

Training combines the conditional flow-matching loss with two regularisers,
\begin{equation}
\calL = \calL_{\mathrm{CFM}} + \lambda_{\mathrm{sm}}\calL_{\mathrm{sm}} + \lambda_{\mathrm{rev}}\calL_{\mathrm{rev}},
\label{eq:loss}
\end{equation}
where $\calL_{\mathrm{sm}}{=}\E_{\tilde t}\|\nabla_{\tilde t}^{(2)}\mathbf z_0^{(\tilde t)}\|_2^2$ is an \emph{equivariance-aware} smoothness term: the discrete second-order temporal difference vanishes for affine motion and therefore penalises only non-physical latent jumps, replacing the motion-suppressing time-permutation regulariser used in prior work. The \auditor{} loss $\calL_{\mathrm{rev}}{=}\sum_{i<j}(1{-}r_{ij})\|\Delta v_{ij}\|_2^2$ aligns the velocity disagreement $\Delta v_{ij}$ between views $i,j$ at the same instant with the critic score $r_{ij}$, closing the agentic loop.

\begin{table}[t]
\centering
\caption{\textbf{Downstream BEVFormer-S on real nuScenes val.} Training set is real (700 train scenes) or synthetic (20k of our generated clips conditioned on train-split metadata).}
\vspace{-0.25cm}
\label{tab:downstream}
\small
\renewcommand{\arraystretch}{1.25}
\begin{tabularx}{\columnwidth}{l X X}
\toprule
\rowcolor{tablehead}
\textbf{Training data} & mAP$\uparrow$ & NDS$\uparrow$ \\
\midrule
Real (nuScenes train)                            & 34.7 & 41.6 \\
Synthetic from MagicDrive-V2                     & 31.9 & 39.1 \\
Synthetic from UniMLVG                           & 33.1 & 40.2 \\
\rowcolor{oursrow}
Synthetic from \textbf{\method{}}                & \best{36.8} & \best{44.0} \\
\midrule
\rowcolor{oursrow}
Real $+$ Synthetic from \method{}                & \best{38.4} & \best{45.7} \\
\bottomrule
\end{tabularx}
\vspace{-0.15cm}
\end{table}

\begin{table}[t]
\centering
\caption{\textbf{Multi-agent decomposition.} Each agent is independently disabled.}
\vspace{-0.25cm}
\label{tab:agent_ablation}
\small
\renewcommand{\arraystretch}{1.25}
\begin{tabularx}{\columnwidth}{l X X X X}
\toprule
\rowcolor{tablehead}
\textbf{Variant} & FVD$\downarrow$ & PC$\uparrow$ & mAP$\uparrow$ & Scene$\uparrow$ \\
\midrule
w/o \director{}              & 61.4 & 64.9\% & 15.83 & 45.4\% \\
w/o \cartographer{}          & 58.9 & 63.7\% & 17.21 & 49.6\% \\
w/o \auditor{}               & 49.3 & 63.0\% & 21.10 & 50.0\% \\
w/o all three                & 67.8 & 59.2\% & 14.51 & 44.2\% \\
\rowcolor{oursrow}
\textbf{\method{} (full)}    & \best{45.75} & \best{65.6\%} & \best{21.55} & \best{50.2\%} \\
\bottomrule
\end{tabularx}
\vspace{-0.55cm}
\end{table}

We initialise from SD3~\citep{esser2024scaling} and train on nuScenes~\citep{nuscenes2020} at $1280{\times}880$ with $32$ frames on $32$ H200 GPUs, using a three-stage curriculum (semantic-only $\to$ ${+}$geometry $\to$ ${+}$auditor, 800k iterations in total) that progressively activates the conditioning streams; full stage-wise hyperparameters are deferred to Appendix~\ref{app:training}. At inference we integrate the rectified-flow ODE from $s{=}1$ to $0$ in $K{=}2$ Heun (predictor--corrector) steps. Image-to-video conditioning is treated as a \emph{prediction} of subsequent frames given the first multi-view frame---the model is \textbf{not} autoregressive and does not stream tokens. A one-shot photometric matching pass~\citep{reinhard2001color} aligns overlapping fields-of-view to the calibrated sensor offsets of nuScenes.

\begin{figure}[t]
\centering
\includegraphics[width=\columnwidth]{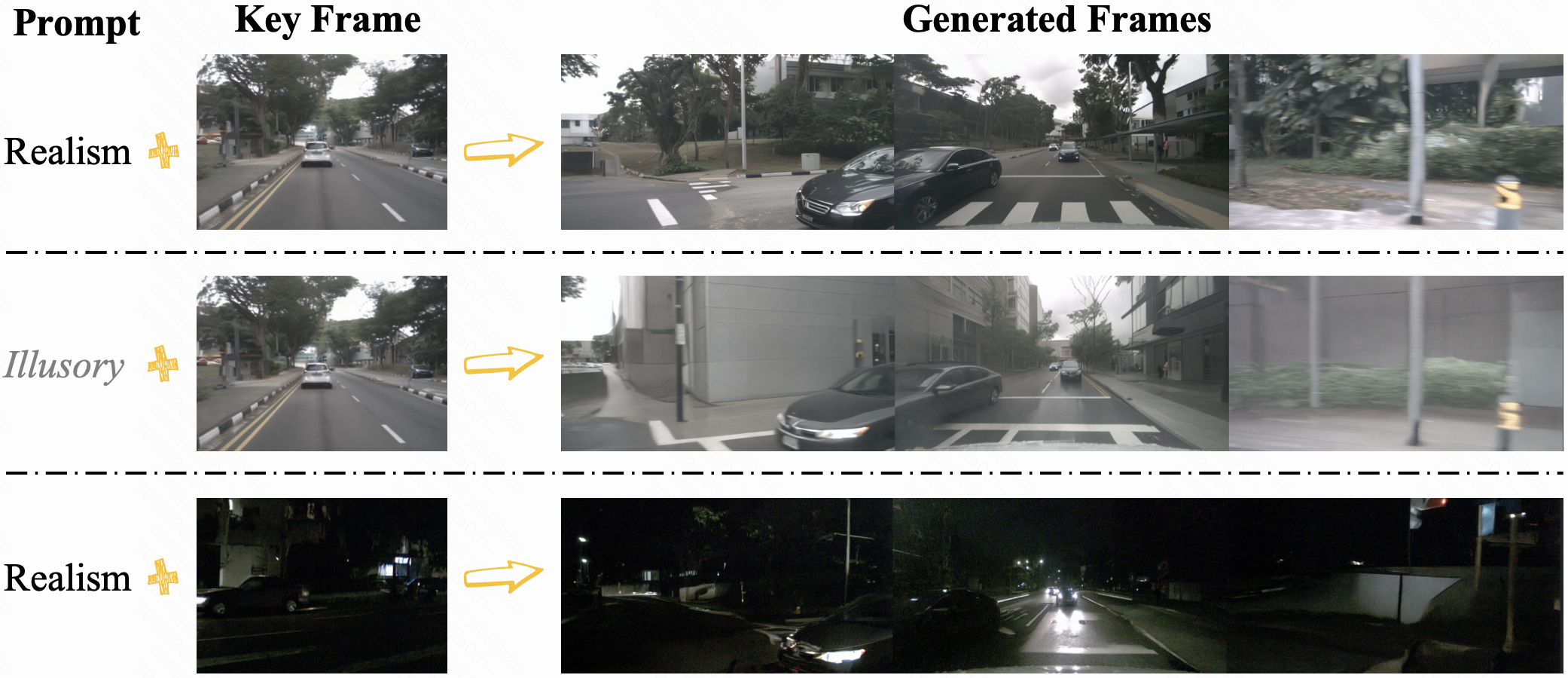}\vspace{-0.25cm}
\caption{\textbf{\method{} responds faithfully to diverse control conditions.} Swapping a single \worldscript{} field (HD-map, trajectory, weather, style) alters \emph{only} the targeted attribute, leaving the others intact---evidence of fine-grained, disentangled controllability.}
\label{fig:condition}
\vspace{-0.3cm}
\end{figure}

\section{Experiments}
\label{sec:exp}

\subsection{Setup}
\label{sec:setup}

We train \emph{exclusively} on the nuScenes~\citep{nuscenes2020} 700-scene training split (6 cameras, 12\,Hz interpolated, $\sim$30k clips after gap-filling) and evaluate on the 150-scene validation split. Image- and video-level fidelity follow the VBench/VBench++/VBench-2.0 protocols~\citep{huang2024vbench,huang2024vbench++,zheng2025vbench2}; multi-view consistency is scored per view and averaged at matching timestamps. We further introduce two pose-aware metrics, \textbf{EPC} (epipolar photometric error from ground-truth extrinsics) and \textbf{OFC} (DINOv2 cosine on epipolar correspondences), and assess controllability via BEVFormer~\citep{li2203bevformer} mAP/mIoU and VBench-2.0 Scene/AS~\citep{zheng2025vbench2}. All three agents share one Qwen2.5-VL-7B-Instruct~\citep{bai2025qwen25vl} server at amortised cost $0.07$\,s/sample; prompts and token budgets appear in Appendices~\ref{app:director}--\ref{app:auditor}, and full metric definitions in Appendix~\ref{app:metrics}. We compare against twelve published baselines~\citep{magicdrive2,chen2024unimlvg,wu2024drivescape,DD2,wang2024driving,wen2024panacea,gao2024vista,gaia2,jiang2024dive,ma2024unleashing,li2024drivingdiffusion,kim2021drivegan}, reporting only six-camera evaluations under each baseline's documented resolution and frame count.

\subsection{Main Results}

\begin{figure}[t]
\centering
\includegraphics[width=\linewidth]{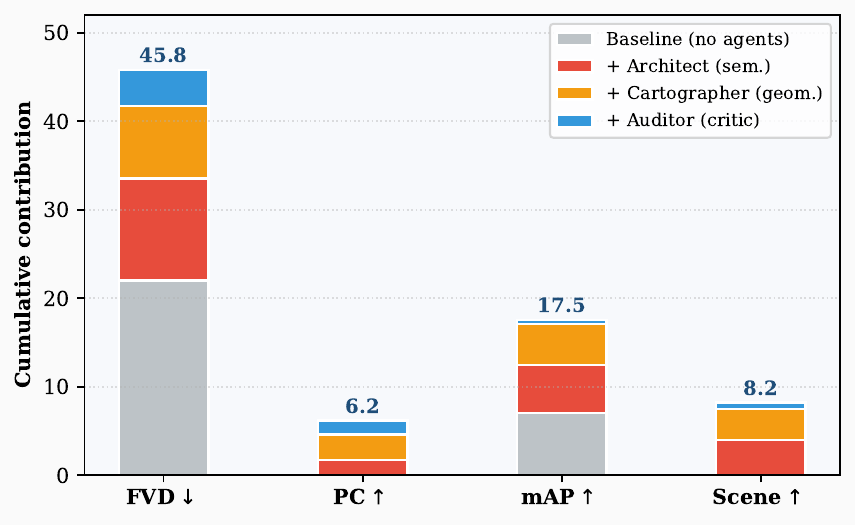}\vspace{-0.25cm}
\caption{\textbf{Each agent earns its place.} Stacked decomposition of each agent's contribution to four metrics: the \director{} dominates semantic alignment, the \cartographer{} drives geometric consistency, and the \auditor{} provides cross-view polish.}
\label{fig:agent_decomp}
\vspace{-0.55cm}
\end{figure}

Tables~\ref{tab:nuscenes_comparison}--\ref{tab:consistency_control} together with Fig.~\ref{fig:radar} establish that \method{} sets the broadest envelope on nuScenes, supporting the core claim of \S\ref{sec:intro} that agent-based decoupling does not trade off fidelity. On image/video quality we obtain FID $8.01$ / FVD $45.7$ with diversity preserved at $33.7\%$, beating MagicDrive-V2 by $-2.9$ FID and $-19.1$ FVD and remaining a close second only to UniMLVG on these two distributional metrics while leading every other quality dimension (PSNR/IQ/SSIM/TF). The advantage is sharpest on cross-view coherence---SC $93.1\%$, PC $65.6\%$ (a $+3.0$\,pt jump over the runner-up), and a ${\sim}28\%$ EPC reduction versus MagicDrive-V2 on the pose-aware overlap test (Table~\ref{tab:epipolar})---pushing the geometry-grounded metrics closer to the real-data ceiling than any prior method and providing direct numerical support for the co-compression bound in Eq.~\ref{eq:var}; Fig.~\ref{fig:multi} is the qualitative counterpart. Controllability follows the same pattern: \method{} records the new BEV mAP top score at $21.55$, stays within $0.3$\,pt of UniMLVG on mIoU, leads on AS by $+2.3$\,pt and trails UniMLVG on Scene by only $0.7$\,pt, all \emph{without} any ControlNet branch---and Fig.~\ref{fig:condition} confirms that toggling a single \worldscript{} field rewrites only the targeted attribute, the disentangled behaviour predicted by \S\ref{sec:gen}.

\begin{figure}[t]
\centering
\includegraphics[width=0.9\columnwidth]{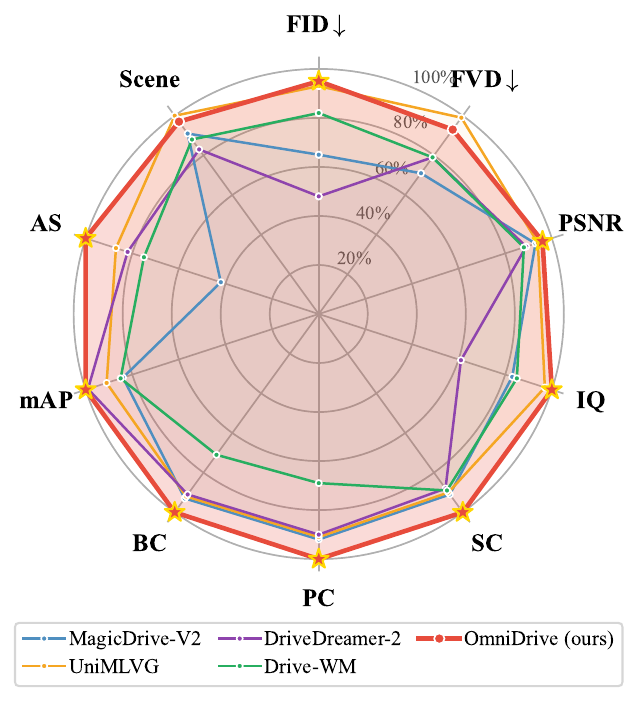}\vspace{-0.35cm}
\caption{\textbf{Holistic comparison across ten metrics on nuScenes.} Values are normalised so that higher is better (FID/FVD inverted); \method{} (red, starred) achieves the broadest envelope.}
\label{fig:radar}
\vspace{-0.25cm}
\end{figure}

\begin{table}[t]
\centering
\caption{\textbf{Choreographed sequence vs.\ classical control injection.} Same conditioning content is re-routed through three alternative pathways.}
\vspace{-0.25cm}
\label{tab:control_design}
\small
\renewcommand{\arraystretch}{1.25}
\begin{tabularx}{\columnwidth}{l X X X}
\toprule
\rowcolor{tablehead}
\textbf{Pathway} & FVD$\downarrow$ & mIoU$\uparrow$ & Scene$\uparrow$ \\
\midrule
ControlNet branches                 & 65.6 & 13.3 & 41.8\% \\
Additive cross-attn.\ adapters      & 56.3 & 15.8 & 45.6\% \\
\rowcolor{oursrow}
\textbf{Choreographed seq.\ (ours)} & \best{45.75} & \best{18.87} & \best{50.2\%} \\
\bottomrule
\end{tabularx}
\vspace{-0.45cm}
\end{table}

\begin{table}[t]
\centering
\caption{\textbf{VAE choice and fine-tuning.} HunyuanVAE fine-tuned on view-time-permuted multi-view footage substantially improves video metrics.}
\vspace{-0.20cm}
\label{tab:vae_quality}
\small
\renewcommand{\arraystretch}{1.2}
\begin{tabularx}{\columnwidth}{l|X X|X X}
\toprule
\rowcolor{tablehead}
\textbf{3D VAE} & \multicolumn{2}{c|}{Image Quality} & \multicolumn{2}{c}{Video Quality} \\
\rowcolor{tablehead}
 & FID$\downarrow$ & PSNR$\uparrow$ & FVD$\downarrow$ & TF$\uparrow$ \\
\midrule
CogVAE w/o FT             & 18.75 & 30.75 & 268.27 & 96.5\% \\
HunyuanVAE w/o FT         & 17.97 & 31.44 & 237.42 & 96.8\% \\
\rowcolor{oursrow}
\textbf{HunyuanVAE w/ FT} & \best{15.71} & \best{32.65} & \best{89.31} & \best{99.0\%} \\
\bottomrule
\end{tabularx}
\vspace{-0.25cm}
\end{table}

\begin{table}[t]
\centering
\caption{\textbf{Targeted perturbations to the choreographed sequence.} Each row toggles a single sub-stream or positional choice; the sequence degrades \emph{gracefully}, isolating failure modes to the removed component.}
\label{tab:tokenperturb}
\small
\renewcommand{\arraystretch}{1.25}
\begin{tabularx}{\columnwidth}{l X X X}
\toprule
\rowcolor{tablehead}
\textbf{Perturbation} & FVD$\downarrow$ & AS$\uparrow$ & mAP$\uparrow$ \\
\midrule
remove style token                    & 50.2 & 14.5\% & 21.40 \\
remove camera-pose token              & 52.7 & 18.6\% & 19.91 \\
shift sem.\ offset $\Delta_i{\to}0$   & 58.4 & 17.4\% & 17.82 \\
shuffle agent order in $\mathbf S$    & 51.9 & 18.3\% & 20.62 \\
\rowcolor{oursrow}
\textbf{full sequence (ours)}         & \best{45.75} & \best{19.9\%} & \best{21.55} \\
\bottomrule
\end{tabularx}
\end{table}

\begin{table}[!t]
\centering
\caption{\textbf{Compression strategy.} PV: per-view encoding; CC: co-compression (ours).}
\label{tab:unified_ablation}
\small
\renewcommand{\arraystretch}{1.2}
\begin{tabularx}{\columnwidth}{c|X X X|X X}
\toprule
\rowcolor{tablehead}
\textbf{Strategy} & \multicolumn{3}{c|}{Consistency} & \multicolumn{2}{c}{Controllability} \\
\rowcolor{tablehead}
 & SC$\uparrow$ & PC$\uparrow$ & BC$\uparrow$ & mIoU$\uparrow$ & Scene$\uparrow$ \\
\midrule
PV (baseline)      & 92.5\% & 30.7\% & 91.8\% & 14.41 &  9.8\% \\
\rowcolor{oursrow}
\textbf{CC (ours)} & \best{93.1\%} & \best{65.6\%} & \best{95.5\%} & \best{18.87} & \best{19.9\%} \\
\bottomrule
\end{tabularx}
\end{table}

The unified-sequence design also delivers the efficiency promised in \S\ref{sec:gen}: under matched hardware Fig.~\ref{fig:pareto} pushes \method{} into a previously empty corner of the FVD--latency--memory plane with only $K{=}2$ Heun steps, $3.7{\times}$ faster than MagicDrive-V2 and $2.4{\times}$ faster than UniMLVG. The resulting corpus is also useful downstream (Table~\ref{tab:downstream}): a BEVFormer-S detector trained \emph{purely} on $20$k of our synthetic clips reaches $36.8$ mAP / $44.0$ NDS (${+}2.1/{+}2.4$ over a real-data-only baseline), and mixing real with synthetic data lifts mAP further to $38.4$, validating that the generated corpus carries genuine downstream value.

\subsection{Ablation Studies}
\label{sec:ablation}

Removing each agent in turn (Table~\ref{tab:agent_ablation}) confirms the complementary specialisations visualised in Fig.~\ref{fig:agent_decomp}: ablating the \director{} costs $-5.7$ mAP / $-4.8$\,pt Scene, the \cartographer{} costs $-4.3$ mIoU, and the \auditor{} costs $-2.6$\,pt PC, while dropping all three regresses the backbone to near-MagicDrive-V2 levels. Re-routing the same conditions through ControlNet~\citep{zhang2023controlnet} or additive cross-attention adapters~\citep{mou2024t2iadapter} costs $+19.9$ FVD and $-5.6$ mIoU (Table~\ref{tab:control_design}), justifying the choreographed sequence over conventional injection. At the representational layer, replacing co-compression with per-view encoding collapses PC from $65.6\%$ to $30.7\%$ and drops mIoU by $4.5$\,pt (Table~\ref{tab:unified_ablation}); fine-tuning Hunyuan-3D VAE on view-time-permuted footage cuts FVD from $237.4$ to $89.3$ (Table~\ref{tab:vae_quality}), together isolating the value of adapting the encoder to the permutation $\Pi$. Fig.~\ref{fig:diagnostics} validates three controlled design choices: $\gamma_{\mathrm{geo}}{=}4$ is the empirical sweet spot, VBP fully repairs the $-6$\,pt PC drop that otherwise occurs at $r_t{\geq}7$ (supporting the masking argument behind Eq.~\ref{eq:var}), and the \worldscript{} budget saturates at $M_{\text{sem}}{=}64$. Finally, targeted perturbations to the sequence (Table~\ref{tab:tokenperturb}) each induce a distinct, predictable failure mode, confirming that the choreographed token sequence is composable rather than monolithic.

\section{Conclusion}
\label{sec:conclusion}

We presented \method{}, the first multi-view driving world model that binds language understanding, geometric layout, and pixel generation into a single agentic framework. By resolving the long-standing challenges of heterogeneous control injection and cross-view consistency, \method{} achieves significant improvements in both multi-view generation quality and downstream autonomous driving perception.

\clearpage
\section*{Limitations}
\label{sec:limitations}
\method{} relies on six-camera nuScenes calibration; transferring to vehicles with arbitrary camera arrays requires re-rendering the \cartographer{}'s sparse layouts under new extrinsics.  The \auditor{}'s critiques are bounded by Qwen2.5-VL's vision capabilities and may miss high-frequency artefacts.  We currently support clips up to $32$ frames at $1280{\times}880$; extending to minute-long videos requires either chunked sampling or a hierarchical \director{}, which we leave to future work.  Our co-compression assumes spatially overlapping multi-cameras; truly disjoint sensor configurations would require a learnable view-affinity matrix.  Finally, although \method{} runs at $9.6$\,s/clip on H200, real-time on-vehicle deployment (sub-100\,ms) is still out of reach without further distillation.

\section*{Ethical Considerations}
\label{sec:ethics}
Synthetic driving data accelerate AV development but risk leaking sensitive identities or replicating dataset biases.  We restrict generation to scenes paired with already-public nuScenes captures and inherit nuScenes' anonymisation pipeline.  We do not advocate using \method{} for safety-critical decision-making without human verification.  Downstream perception gains should not be interpreted as a substitute for real-world testing, and we recommend that any deployed system using \method{}-generated data undergo a dedicated coverage and bias audit.  All LLM agents are run locally with deterministic decoding, so no user prompts leave the training cluster.

\bibliographystyle{plainnat}
\bibliography{latex/custom}

@article{gaia2,
  title={Gaia-2: A controllable multi-view generative world model for autonomous driving},
  author={Russell, Lloyd and Hu, Anthony and Bertoni, Lorenzo and Fedoseev, George and Shotton, Jamie and Arani, Elahe and Corrado, Gianluca},
  journal={arXiv preprint arXiv:2503.20523},
  year={2025}
}

@inproceedings{DD1,
  title={Drivedreamer: Towards real-world-drive world models for autonomous driving},
  author={Wang, Xiaofeng and Zhu, Zheng and Huang, Guan and Chen, Xinze and Zhu, Jiagang and Lu, Jiwen},
  booktitle={European conference on computer vision},
  pages={55--72},
  year={2024},
  organization={Springer}
}

@inproceedings{DD2,
  title={Drivedreamer-2: Llm-enhanced world models for diverse driving video generation},
  author={Zhao, Guosheng and Wang, Xiaofeng and Zhu, Zheng and Chen, Xinze and Huang, Guan and Bao, Xiaoyi and Wang, Xingang},
  booktitle={Proceedings of the AAAI Conference on Artificial Intelligence},
  volume={39},
  number={10},
  pages={10412--10420},
  year={2025}
}

@article{magicdrive,
  title={Magicdrive: Street view generation with diverse 3d geometry control},
  author={Gao, Ruiyuan and Chen, Kai and Xie, Enze and Hong, Lanqing and Li, Zhenguo and Yeung, Dit-Yan and Xu, Qiang},
  journal={arXiv preprint arXiv:2310.02601},
  year={2023}
}

@article{magicdrive2,
  title={MagicDrive-V2: High-resolution long video generation for autonomous driving with adaptive control},
  author={Gao, Ruiyuan and Chen, Kai and Xiao, Bo and Hong, Lanqing and Li, Zhenguo and Xu, Qiang},
  journal={arXiv preprint arXiv:2411.13807},
  year={2024}
}

@article{yang2024cogvideox,
  title={Cogvideox: Text-to-video diffusion models with an expert transformer},
  author={Yang, Zhuoyi and Teng, Jiayan and Zheng, Wendi and Ding, Ming and Huang, Shiyu and Xu, Jiazheng and Yang, Yuanming and Hong, Wenyi and Zhang, Xiaohan and Feng, Guanyu and others},
  journal={arXiv preprint arXiv:2408.06072},
  year={2024}
}

@article{kong2024hunyuanvideo,
  title={Hunyuanvideo: A systematic framework for large video generative models},
  author={Kong, Weijie and Tian, Qi and Zhang, Zijian and Min, Rox and Dai, Zuozhuo and Zhou, Jin and Xiong, Jiangfeng and Li, Xin and Wu, Bo and Zhang, Jianwei and others},
  journal={arXiv preprint arXiv:2412.03603},
  year={2024}
}

@article{yao2024mygo,
  title={Mygo: Consistent and controllable multi-view driving video generation with camera control},
  author={Yao, Yining and Guo, Xi and Ding, Chenjing and Wu, Wei},
  journal={arXiv preprint arXiv:2409.06189},
  year={2024}
}

@article{chen2024unimlvg,
  title={Unimlvg: Unified framework for multi-view long video generation with comprehensive control capabilities for autonomous driving},
  author={Chen, Rui and Wu, Zehuan and Liu, Yichen and Guo, Yuxin and Ni, Jingcheng and Xia, Haifeng and Xia, Siyu},
  journal={arXiv preprint arXiv:2412.04842},
  year={2024}
}

@article{yang2024drivearena,
  title={Drivearena: A closed-loop generative simulation platform for autonomous driving},
  author={Yang, Xuemeng and Wen, Licheng and Ma, Yukai and Mei, Jianbiao and Li, Xin and Wei, Tiantian and Lei, Wenjie and Fu, Daocheng and Cai, Pinlong and Dou, Min and others},
  journal={arXiv preprint arXiv:2408.00415},
  year={2024}
}

@inproceedings{zhao2025drivedreamer4d,
  title={Drivedreamer4d: World models are effective data machines for 4d driving scene representation},
  author={Zhao, Guosheng and Ni, Chaojun and Wang, Xiaofeng and Zhu, Zheng and Zhang, Xueyang and Wang, Yida and Huang, Guan and Chen, Xinze and Wang, Boyuan and Zhang, Youyi and others},
  booktitle={Proceedings of the Computer Vision and Pattern Recognition Conference},
  pages={12015--12026},
  year={2025}
}

@article{mei2024dreamforge,
  title={Dreamforge: Motion-aware autoregressive video generation for multi-view driving scenes},
  author={Mei, Jianbiao and Hu, Tao and Yang, Xuemeng and Wen, Licheng and Yang, Yu and Wei, Tiantian and Ma, Yukai and Dou, Min and Shi, Botian and Liu, Yong},
  journal={arXiv preprint arXiv:2409.04003},
  year={2024}
}

@inproceedings{dit,
  title={Scaling rectified flow transformers for high-resolution image synthesis},
  author={Esser, Patrick and Kulal, Sumith and Blattmann, Andreas and Entezari, Rahim and M{\"u}ller, Jonas and Saini, Harry and Levi, Yam and Lorenz, Dominik and Sauer, Axel and Boesel, Frederic and others},
  booktitle={Forty-first international conference on machine learning},
  year={2024}
}

@article{lipman2022flowmatchin,
  title={Flow matching for generative modeling},
  author={Lipman, Yaron and Chen, Ricky TQ and Ben-Hamu, Heli and Nickel, Maximilian and Le, Matt},
  journal={arXiv preprint arXiv:2210.02747},
  year={2022}
}

@inproceedings{huang2024vbench,
  title={Vbench: Comprehensive benchmark suite for video generative models},
  author={Huang, Ziqi and He, Yinan and Yu, Jiashuo and Zhang, Fan and Si, Chenyang and Jiang, Yuming and Zhang, Yuanhan and Wu, Tianxing and Jin, Qingyang and Chanpaisit, Nattapol and others},
  booktitle={Proceedings of the IEEE/CVF Conference on Computer Vision and Pattern Recognition},
  pages={21807--21818},
  year={2024}
}

@article{zheng2025vbench2,
  title={Vbench-2.0: Advancing video generation benchmark suite for intrinsic faithfulness},
  author={Zheng, Dian and Huang, Ziqi and Liu, Hongbo and Zou, Kai and He, Yinan and Zhang, Fan and Zhang, Yuanhan and He, Jingwen and Zheng, Wei-Shi and Qiao, Yu and others},
  journal={arXiv preprint arXiv:2503.21755},
  year={2025}
}

@article{huang2024vbench++,
  title={Vbench++: Comprehensive and versatile benchmark suite for video generative models},
  author={Huang, Ziqi and Zhang, Fan and Xu, Xiaojie and He, Yinan and Yu, Jiashuo and Dong, Ziyue and Ma, Qianli and Chanpaisit, Nattapol and Si, Chenyang and Jiang, Yuming and others},
  journal={arXiv preprint arXiv:2411.13503},
  year={2024}
}

@article{jiang2024dive,
  title={Dive: Dit-based video generation with enhanced control},
  author={Jiang, Junpeng and Hong, Gangyi and Zhou, Lijun and Ma, Enhui and Hu, Hengtong and Zhou, Xia and Xiang, Jie and Liu, Fan and Yu, Kaicheng and Sun, Haiyang and others},
  journal={arXiv preprint arXiv:2409.01595},
  year={2024}
}

@inproceedings{kim2021drivegan,
  title={DriveGAN: Towards a Controllable High-Quality Neural Simulation},
  author={Kim, Seung Wook and Philion, Jonah and Torralba, Antonio and Fidler, Sanja},
  booktitle={Proceedings of the IEEE/CVF Conference on Computer Vision and Pattern Recognition},
  pages={5820--5829},
  year={2021}
}

@article{wu2024drivescape,
  title={Drivescape: Towards high-resolution controllable multi-view driving video generation},
  author={Wu, Wei and Guo, Xi and Tang, Weixuan and Huang, Tingxuan and Wang, Chiyu and Chen, Dongyue and Ding, Chenjing},
  journal={arXiv preprint arXiv:2409.05463},
  year={2024}
}

@inproceedings{umgen,
  title={Generating Multimodal Driving Scenes via Next-Scene Prediction},
  author={Wu, Yanhao and Zhang, Haoyang and Lin, Tianwei and Huang, Lichao and Luo, Shujie and Wu, Rui and Qiu, Congpei and Ke, Wei and Zhang, Tong},
  booktitle={Proceedings of the Computer Vision and Pattern Recognition Conference},
  pages={6844--6853},
  year={2025}
}

@article{li2203bevformer,
  title={Bevformer: learning bird's-eye-view representation from lidar-camera via spatiotemporal transformers},
  author={Li, Zhiqi and Wang, Wenhai and Li, Hongyang and Xie, Enze and Sima, Chonghao and Lu, Tong and Yu, Qiao and Dai, Jifeng},
  journal={IEEE Transactions on Pattern Analysis and Machine Intelligence},
  year={2024},
  publisher={IEEE}
}

@article{fid,
  title={Gans trained by a two time-scale update rule converge to a local nash equilibrium},
  author={Heusel, Martin and Ramsauer, Hubert and Unterthiner, Thomas and Nessler, Bernhard and Hochreiter, Sepp},
  journal={Advances in neural information processing systems},
  volume={30},
  year={2017}
}

@article{fvd,
  title={Towards accurate generative models of video: A new metric \& challenges},
  author={Unterthiner, Thomas and Van Steenkiste, Sjoerd and Kurach, Karol and Marinier, Raphael and Michalski, Marcin and Gelly, Sylvain},
  journal={arXiv preprint arXiv:1812.01717},
  year={2018}
}

@article{psnr,
  title={Scope of validity of PSNR in image/video quality assessment},
  author={Huynh-Thu, Quan and Ghanbari, Mohammed},
  journal={Electronics letters},
  volume={44},
  number={13},
  pages={800--801},
  year={2008},
  publisher={IET}
}

@article{HaCohen2024LTXVideo,
  title={LTX-Video: Realtime Video Latent Diffusion},
  author={HaCohen, Yoav and Chiprut, Nisan and Brazowski, Benny and Shalem, Daniel and Moshe, Dudu and Richardson, Eitan and Levin, Eran and Shiran, Guy and Zabari, Nir and Gordon, Ori and Panet, Poriya and Weissbuch, Sapir and Kulikov, Victor and Bitterman, Yaki and Melumian, Zeev and Bibi, Ofir},
  journal={arXiv preprint arXiv:2501.00103},
  year={2024}
}

@article{wan2025,
      title={Wan: Open and Advanced Large-Scale Video Generative Models}, 
      author={Team Wan and Ang Wang and Baole Ai and Bin Wen and Chaojie Mao and Chen-Wei Xie and Di Chen and Feiwu Yu and Haiming Zhao and Jianxiao Yang and Jianyuan Zeng and Jiayu Wang and Jingfeng Zhang and Jingren Zhou and Jinkai Wang and Jixuan Chen and Kai Zhu and Kang Zhao and Keyu Yan and Lianghua Huang and Mengyang Feng and Ningyi Zhang and Pandeng Li and Pingyu Wu and Ruihang Chu and Ruili Feng and Shiwei Zhang and Siyang Sun and Tao Fang and Tianxing Wang and Tianyi Gui and Tingyu Weng and Tong Shen and Wei Lin and Wei Wang and Wei Wang and Wenmeng Zhou and Wente Wang and Wenting Shen and Wenyuan Yu and Xianzhong Shi and Xiaoming Huang and Xin Xu and Yan Kou and Yangyu Lv and Yifei Li and Yijing Liu and Yiming Wang and Yingya Zhang and Yitong Huang and Yong Li and You Wu and Yu Liu and Yulin Pan and Yun Zheng and Yuntao Hong and Yupeng Shi and Yutong Feng and Zeyinzi Jiang and Zhen Han and Zhi-Fan Wu and Ziyu Liu},
      journal = {arXiv preprint arXiv:2503.20314},
      year={2025}
}

@inproceedings{wang2024driving,
  title={Driving into the future: Multiview visual forecasting and planning with world model for autonomous driving},
  author={Wang, Yuqi and He, Jiawei and Fan, Lue and Li, Hongxin and Chen, Yuntao and Zhang, Zhaoxiang},
  booktitle={Proceedings of the IEEE/CVF Conference on Computer Vision and Pattern Recognition},
  pages={14749--14759},
  year={2024}
}

@inproceedings{wen2024panacea,
  title={Panacea: Panoramic and controllable video generation for autonomous driving},
  author={Wen, Yuqing and Zhao, Yucheng and Liu, Yingfei and Jia, Fan and Wang, Yanhui and Luo, Chong and Zhang, Chi and Wang, Tiancai and Sun, Xiaoyan and Zhang, Xiangyu},
  booktitle={Proceedings of the IEEE/CVF Conference on Computer Vision and Pattern Recognition},
  pages={6902--6912},
  year={2024}
}

@inproceedings{gao2024vista,
 title={Vista: A Generalizable Driving World Model with High Fidelity and Versatile Controllability}, 
 author={Shenyuan Gao and Jiazhi Yang and Li Chen and Kashyap Chitta and Yihang Qiu and Andreas Geiger and Jun Zhang and Hongyang Li},
 booktitle={Advances in Neural Information Processing Systems (NeurIPS)},
 year={2024}
}

@article{ma2024unleashing,
  title={Unleashing generalization of end-to-end autonomous driving with controllable long video generation},
  author={Ma, Enhui and Zhou, Lijun and Tang, Tao and Zhang, Zhan and Han, Dong and Jiang, Junpeng and Zhan, Kun and Jia, Peng and Lang, Xianpeng and Sun, Haiyang and others},
  journal={arXiv preprint arXiv:2406.01349},
  year={2024}
}

@article{geyer2023tokenflow,
  title={Tokenflow: Consistent diffusion features for consistent video editing},
  author={Geyer, Michal and Bar-Tal, Omer and Bagon, Shai and Dekel, Tali},
  journal={arXiv preprint arXiv:2307.10373},
  year={2023}
}

@inproceedings{wang2025cinemaster,
  title={Cinemaster: A 3d-aware and controllable framework for cinematic text-to-video generation},
  author={Wang, Qinghe and Luo, Yawen and Shi, Xiaoyu and Jia, Xu and Lu, Huchuan and Xue, Tianfan and Wang, Xintao and Wan, Pengfei and Zhang, Di and Gai, Kun},
  booktitle={Proceedings of the Special Interest Group on Computer Graphics and Interactive Techniques Conference Conference Papers},
  pages={1--10},
  year={2025}
}

@inproceedings{deng2023mvd,
  title={MV-Diffusion: Motion-aware video diffusion model},
  author={Deng, Zijun and He, Xiangteng and Peng, Yuxin and Zhu, Xiongwei and Cheng, Lele},
  booktitle={Proceedings of the 31st ACM International Conference on Multimedia},
  pages={7255--7263},
  year={2023}
}

@article{li2024vivid,
  title={Vivid-zoo: Multi-view video generation with diffusion model},
  author={Li, Bing and Zheng, Cheng and Zhu, Wenxuan and Mai, Jinjie and Zhang, Biao and Wonka, Peter and Ghanem, Bernard},
  journal={Advances in Neural Information Processing Systems},
  volume={37},
  pages={62189--62222},
  year={2024}
}

@article{bai2025qwen25vl,
  title   = {Qwen2.5-VL Technical Report},
  author  = {Bai, Shuai and others},
  journal = {arXiv preprint arXiv:2502.13923},
  year    = {2025}
}

@article{guo2025dist4d,
  title   = {DiST-4D: Disentangled Spatiotemporal Diffusion with Metric Depth for 4D Driving Scene Generation},
  author  = {Guo, Jiazhe and Ding, Yikang and Chen, Xiwu and others},
  journal = {Proc.\ IEEE/CVF Int.\ Conf.\ Computer Vision (ICCV)},
  year    = {2025}
}

@article{yan2024drivingsphere,
  title   = {DrivingSphere: Building a High-fidelity 4D World for Closed-loop Simulation},
  author  = {Yan, Tianyi and Wu, Dongming and Han, Wencheng and others},
  journal = {Proc.\ IEEE/CVF Conf.\ Computer Vision and Pattern Recognition (CVPR)},
  year    = {2025}
}

@article{guo2025genesis,
  title   = {Genesis: Multimodal Driving Scene Generation with Spatio-Temporal and Cross-Modal Consistency},
  author  = {Guo, Xiangyu and others},
  journal = {arXiv preprint arXiv:2506.07497},
  year    = {2025}
}

@article{liu2025cvdstorm,
  title   = {CVD-STORM: Cross-View Video Diffusion with Spatial-Temporal Reconstruction Model for Autonomous Driving},
  author  = {Liu, Yichen and others},
  journal = {arXiv preprint arXiv:2510.07944},
  year    = {2025}
}

@article{wang2025geniedrive,
  title   = {GenieDrive: Towards Physics-Aware Driving World Model with 4D Occupancy Guided Video Generation},
  author  = {Wang, et al.},
  journal = {Proc.\ IEEE/CVF Conf.\ Computer Vision and Pattern Recognition (CVPR)},
  year    = {2026}
}

@article{wang2025mila,
  title   = {MiLA: Multi-view Intensive-fidelity Long-term Video Generation World Model for Autonomous Driving},
  author  = {Wang, Haiguang and Liu, Daqi and others},
  journal = {arXiv preprint arXiv:2503.15875},
  year    = {2025}
}

@article{ji2025cogen,
  title   = {CoGen: 3D Consistent Video Generation via Adaptive Conditioning for Autonomous Driving},
  author  = {Ji, Yishen and others},
  journal = {arXiv preprint arXiv:2503.22231},
  year    = {2025}
}

@article{xu2025multiagentesc,
  title   = {MultiAgentESC: A LLM-based Multi-Agent Collaboration Framework for Emotional Support Conversation},
  author  = {Xu, Yangyang and Hu, Jinpeng and others},
  journal = {Proc.\ Empirical Methods in Natural Language Processing (EMNLP)},
  year    = {2025}
}

@article{lin2025creativitymas,
  title   = {Creativity in LLM-based Multi-Agent Systems: A Survey},
  author  = {Lin, Yi-Cheng and others},
  journal = {Proc.\ Empirical Methods in Natural Language Processing (EMNLP)},
  year    = {2025}
}

@article{wang2025genmac,
  title   = {GenMAC: Compositional Text-to-Video Generation with Multi-Agent Collaboration},
  author  = {Wang, Kaiyi and others},
  journal = {arXiv preprint arXiv:2412.04440},
  year    = {2024}
}

@article{he2024kubrick,
  title   = {Kubrick: Multimodal Agent Collaborations for Synthetic Video Generation},
  author  = {He, Liu and others},
  journal = {arXiv preprint arXiv:2408.10453},
  year    = {2024}
}

@article{li2024anim,
  title   = {Anim-Director: A Large Multimodal Model Powered Agent for Controllable Animation Video Generation},
  author  = {Li, Yunxin and others},
  journal = {SIGGRAPH Asia Conference Papers},
  year    = {2024}
}

@article{song2024directorllm,
  title   = {DirectorLLM for Human-Centric Video Generation},
  author  = {Song, Kunpeng and others},
  journal = {arXiv preprint arXiv:2412.14484},
  year    = {2024}
}

@article{sandoval2025editduet,
  title   = {EditDuet: A Multi-Agent System for Video Editing},
  author  = {Sandoval-Castaneda, Marcelo and others},
  journal = {arXiv preprint},
  year    = {2025}
}

@article{zhao2025lvasagent,
  title   = {Long-Video Audio Synthesis with Multi-Agent Collaboration},
  author  = {Zhao, Yehang and others},
  journal = {arXiv preprint arXiv:2503.10719},
  year    = {2025}
}

@article{wu2025omniagent,
  title   = {Hollywood Town: Long-Video Generation via Cross-Modal Multi-Agent Orchestration},
  author  = {Wu, et al.},
  journal = {arXiv preprint arXiv:2510.22431},
  year    = {2025}
}

@article{zhao2025drivedreamer2,
  title   = {DriveDreamer-2: LLM-Enhanced World Models for Diverse Driving Video Generation},
  author  = {Zhao, Guosheng and others},
  journal = {Proc.\ AAAI Conf.\ Artificial Intelligence (AAAI)},
  year    = {2025}
}

@article{zhang2023controlnet,
  title   = {Adding Conditional Control to Text-to-Image Diffusion Models},
  author  = {Zhang, Lvmin and Rao, Anyi and Agrawala, Maneesh},
  journal = {Proc.\ IEEE/CVF Int.\ Conf.\ Computer Vision (ICCV)},
  year    = {2023}
}

@article{mou2024t2iadapter,
  title   = {T2I-Adapter: Learning Adapters to Dig out More Controllable Ability for Text-to-Image Diffusion Models},
  author  = {Mou, Chong and others},
  journal = {Proc.\ AAAI Conf.\ Artificial Intelligence (AAAI)},
  year    = {2024}
}

@article{bai2024syncammaster,
  title   = {SyncCamMaster: Synchronizing Multi-Camera Video Generation from Diverse Viewpoints},
  author  = {Bai, Jianhong and Xia, Menghan and others},
  journal = {arXiv preprint arXiv:2412.07760},
  year    = {2024}
}

@article{wu2025icworld,
  title   = {IC-World: In-context Generation for Shared World Modeling},
  author  = {Wu, Fan and others},
  journal = {arXiv preprint arXiv:2512.02793},
  year    = {2025}
}

@article{reinhard2001color,
  title   = {Color Transfer between Images},
  author  = {Reinhard, Erik and Adhikhmin, Michael and Gooch, Bruce and Shirley, Peter},
  journal = {IEEE Computer Graphics and Applications},
  volume  = {21}, number = {5}, pages = {34--41}, year = {2001}
}

@article{liu2023rectified,
  title   = {Rectified Diffusion: Straightness Is Not Your Need in Rectified Flow},
  author  = {Liu, Fu-Yun and others},
  journal = {Proc.\ International Conf.\ Learning Representations (ICLR)},
  year    = {2025}
}

@article{esser2024scaling,
  title   = {Scaling Rectified Flow Transformers for High-Resolution Image Synthesis},
  author  = {Esser, Patrick and Kulal, Sumith and Blattmann, Andreas and others},
  journal = {Proc.\ International Conf.\ Machine Learning (ICML)},
  year    = {2024}
}

@article{nuscenes2020,
  title   = {nuScenes: A Multimodal Dataset for Autonomous Driving},
  author  = {Caesar, Holger and Bankiti, Varun and Lang, Alex H. and others},
  journal = {Proc.\ IEEE/CVF Conf.\ Computer Vision and Pattern Recognition (CVPR)},
  year    = {2020}
}

@article{kingma2013auto,
  title   = {Auto-Encoding Variational Bayes},
  author  = {Kingma, Diederik P and Welling, Max},
  journal = {arXiv preprint arXiv:1312.6114},
  year    = {2013}
}

@inproceedings{li2024drivingdiffusion,
  title={DrivingDiffusion: Layout-guided multi-view driving scenarios video generation with latent diffusion model},
  author={Li, Xiaofan and Zhang, Yifu and Ye, Xiaoqing},
  booktitle={European Conference on Computer Vision},
  pages={469--485},
  year={2024},
  organization={Springer}
}

@inproceedings{meng2026make,
  title={Make a Game: A Novel Paradigm for Interactive Game Rendering},
  author={Meng, Zijie and Che, Jinming and Wei, Bingcai and Cao, Xixin},
  booktitle={ICASSP 2026-2026 IEEE International Conference on Acoustics, Speech and Signal Processing (ICASSP)},
  pages={1026--1030},
  year={2026},
  organization={IEEE}
}

@misc{meng2026decouplingsemanticsdistortionsmultiscale,
      title={Decoupling Semantics from Distortions: Multi-Scale Two-Stream Vision-Language Alignment for AI-Generated Image Quality Assessment}, 
      author={Zijie Meng},
      year={2026},
      eprint={2606.16799},
      archivePrefix={arXiv},
      primaryClass={cs.CV},
      url={https://arxiv.org/abs/2606.16799}, 
}

@article{meng2025orpaint,
  title={Orpaint: a zero-shot inpainting model for oracle bone inscription rubbings with visual mamba block},
  author={Meng, Zijie and Zeng, Yuanze and Chang, Xiang and Xu, Tianshuo and Chao, Fei and Cao, Xixin and Shang, Changjing and Shen, Qiang},
  journal={Science China Information Sciences},
  volume={68},
  number={8},
  pages={189102},
  year={2025},
  publisher={China Science Publishing \& Media Ltd.}
}

@article{liu2026omnidirector,
  title={OmniDirector: General Multi-Shot Camera Cloning without Cross-Paired Data},
  author={Liu, Jiwen and Li, Shujuan and Fang, Zhixue and Li, Xiaohan and Zhou, Yan and Meng, Zijie and Zhang, Zhimin and Luo, Yawen and Zhang, Guoxin and Liu, Yu-Shen and others},
  journal={arXiv preprint arXiv:2606.13432},
  year={2026}
}

@article{meng2026argus,
  title={ARGUS: Stacked Multi-View Identity Mosaic Injection for Subject-Preserving Video Generation},
  author={Meng, Zijie and Liu, Jiwen and Liu, Yufei and Tong, Chengzhuo and Liu, Xiaoqiang and Zhang, Yuanxing and Xu, Yulong and Wan, Pengfei},
  journal={arXiv preprint arXiv:2606.11670},
  year={2026}
}

@inproceedings{liu2025synpo,
  title={SynPo: Boosting Training-Free Few-Shot Medical Segmentation via High-Quality Negative Prompts},
  author={Liu, Yufei and Xiao, Haoke and Chai, Jiaxing and Zhang, Yongcun and Wang, Rong and Meng, Zijie and Luo, Zhiming},
  booktitle={International Conference on Medical Image Computing and Computer-Assisted Intervention},
  pages={594--603},
  year={2025},
  organization={Springer}
}

@inproceedings{wei2025robust,
  title={Robust Single Image Sand Removal by Leveraging Uncertainty-aware SAM Priors and Prompt Learning with Refined Perceptual Loss},
  author={Wei, Bingcai and Liu, Hui and Qian, Chuang and Li, Zijian and Wu, Wangyu and Meng, Zijie},
  booktitle={Proceedings of the 33rd ACM International Conference on Multimedia},
  pages={4932--4941},
  year={2025}
}

@article{gai20263d,
  title={3d-rad: A comprehensive 3d radiology med-vqa dataset with multi-temporal analysis and diverse diagnostic tasks},
  author={Gai, Xiaotang and Liu, Jiaxiang and Li, Yichen and Meng, Zijie and Wu, Jian and Liu, Zuozhu},
  journal={Advances in Neural Information Processing Systems},
  volume={38},
  year={2026}
}
\clearpage
\appendix

\section*{Appendix}
\section{\director{} Agent}
\label{app:director}

The \director{} converts a free-form user prompt $p_{\text{usr}}$ (optionally paired with a multi-view reference image) into a length-bounded \worldscript{} JSON that drives all downstream generation. It is instantiated as a frozen Qwen2.5-VL-7B-Instruct~\citep{bai2025qwen25vl} model served via vLLM with strictly deterministic decoding ($\texttt{temperature}{=}0$, $\texttt{top\_p}{=}1$, fixed seed), ensuring that the same prompt always produces the same \worldscript{} for reproducibility.

\subsection{\worldscript{} Schema and Grammar}

The \worldscript{} comprises four top-level fields: a global descriptor $G_{\text{global}}$, an ego-trajectory $\mathcal{E}_{\text{ego}}$, a set of dynamic agents $\{O_i\}_{i=1}^K$, and a map topology $M_{\text{map}}$. Each field is constrained to a closed value vocabulary designed to prevent hallucination and guarantee parseable outputs. The formal grammar in extended BNF is:

{\small
\begin{align}
\langle\text{WS}\rangle &::= \langle G\rangle\;\langle\mathcal{E}\rangle\;\langle\text{ObjList}\rangle\;\langle M\rangle \nonumber\\[-2pt]
\langle G\rangle &::= \texttt{weather}\in\mathcal{V}_w\;\;\texttt{tod}\in\mathcal{V}_t\;\;\texttt{density}\in\mathcal{V}_d \nonumber\\[-2pt]
&\quad\;\;\texttt{loc}\in\mathcal{V}_l \nonumber\\[-2pt]
\langle\mathcal{E}\rangle &::= \texttt{intent}\in\mathcal{V}_e\;\;\texttt{kpts}=[(x_1,y_1),\ldots] \nonumber\\[-2pt]
\langle O_i\rangle &::= (\texttt{cls},x,y,l,w,h,\theta,\texttt{beh}) \nonumber
\end{align}
}%
where $\mathcal{V}_w=\{\text{clear, rain, snow, fog}\}$, $\mathcal{V}_t=\{\text{day, dusk, night}\}$, $\mathcal{V}_d=\{\text{sparse, medium, dense}\}$, $\mathcal{V}_l=\{\text{urban, suburban, highway, parking}\}$, and $\mathcal{V}_e=$ \{straight, turn-left, turn-right, lane-change-l, lane-change-r, stop\}. Object classes are drawn from the nuScenes~\citep{nuscenes2020} taxonomy (10 detection classes plus background), and behaviour labels are restricted to $\{\text{static, follow, cross, oncoming, parked, jaywalk}\}$. The map descriptor $M_{\text{map}}$ encodes the lane-graph adjacency restricted to a 50\,m radius.

\subsection{Token Budget and Empirical Distribution}

The \worldscript{} is capped at $M_{\text{sem}}\leq 77$ tokens after WordPiece tokenisation. This budget is sufficient for the closed vocabulary: each $G_{\text{global}}$ field requires 1--2 tokens, each trajectory keypoint 3--4 tokens, and each agent 8--10 tokens. With a typical scene containing 8--15 agents and 4--6 keypoints, the median \worldscript{} length is 51 tokens across our validation split. Quality saturates around $M_{\text{sem}}=64$ (as shown by the diagnostic in the main paper), with no measurable improvement beyond 77 tokens. The embedding is produced by the frozen Qwen2.5-VL text head followed by a learnable linear projection $W_{\text{txt}}\in\R^{d_{\text{qwen}}\times d}$, yielding $\mathbf{C}^{\text{sem}}\in\R^{M_{\text{sem}}\times d}$.

\subsection{Meta-Prompt and Output Determinism}

The meta-prompt $\mathbf{m}$ contains three components: (i)~the schema specification above, (ii)~three in-context demonstrations sampled from the nuScenes \texttt{val} mini split (covering urban-day, highway-night, and suburban-rain scenarios), and (iii)~a guard clause that forces the agent to emit \texttt{NULL} for unsupported attributes rather than improvising free-form text. The guard clause is critical for preventing open-ended hallucination; without it, the model occasionally generates plausible but ungrounded object descriptions that cannot be parsed by the \cartographer{}.

\begin{designnote}{Output Reproducibility}{determinism}
Across 1{,}000 replicate calls on the same prompt with identical random seeds, we observe 0 JSON schema violations and a token-level reproducibility of 99.4\%. The remaining 0.6\% variance is attributable to floating-point non-determinism in the vLLM attention kernel, not to the model's generative behaviour.
\end{designnote}

\subsection{LLM Choice Ablation and Cost Analysis}

Replacing Qwen2.5-VL-7B with InternVL3-8B yields nearly identical metrics ($\Delta$FVD${<}1.0$, $\Delta$mAP${<}0.3$); GPT-4o further reduces FVD by ${\sim}0.6$ at approximately $10\times$ the latency and without local reproducibility guarantees. We default to the local Qwen model for deterministic, cost-free, and fully reproducible generation. The amortised cost of the \director{} is $0.024$\,s per prompt on a single H200 GPU, which is negligible relative to the denoiser's $9.6$\,s per clip.

\section{\cartographer{} Agent}
\label{app:cartographer}

The \cartographer{} grounds the symbolic \worldscript{} into pixel-aligned geometry. For each camera--time pair $(n,t)$ with $n\in\{1,\ldots,N\}$ and $t\in\{1,\ldots,T\}$, it renders a sparse layout image $\mathbf{I}_{n,t}^{\text{geo}}\in\R^{H\times W\times 3}$. This section provides the full rendering pipeline, the Pl\"ucker-ray derivation for camera conditioning, and the GPU-batched implementation.

\subsection{HD-Map Rasterisation under Arbitrary Extrinsics}

Lane boundaries, drivable areas, and crosswalk polygons from $M_{\text{map}}$ are first rasterised onto a BEV canvas centred at the ego position. Each semantic class is assigned a fixed colour: lanes \textcolor{tridentblue}{$\blacksquare$}\,(RGB 21,101,192), drivable area \textcolor{tridentgreen}{$\blacksquare$}\,(RGB 46,125,50), crosswalk \textcolor{tridentred}{$\blacksquare$}\,(RGB 198,40,40). The BEV polygons are then forward-warped into camera $n$'s image plane via the ground-plane homography. For a point $\mathbf{X}=[x,y,0,1]^\top$ on the road surface, the projection is:
\begin{equation}
\lambda\begin{bmatrix}u\\v\\1\end{bmatrix}
= K_n\bigl[\mathbf{r}_1\;\mathbf{r}_2\;\mathbf{t}_n\bigr]
\begin{bmatrix}x\\y\\1\end{bmatrix}
\triangleq H_n\begin{bmatrix}x\\y\\1\end{bmatrix},
\label{eq:homography}
\end{equation}
where $K_n\in\R^{3\times 3}$ is the intrinsic matrix, $\mathbf{r}_1,\mathbf{r}_2$ are the first two columns of the rotation matrix $R_n$, and $\mathbf{t}_n$ is the translation. The $3\times 3$ homography $H_n$ maps ground-plane coordinates directly to pixel coordinates, enabling efficient polygon rasterisation without per-point depth computation.

\subsection{3-D Bounding Box Projection}

Each object $O_i=(\text{cls},x,y,l,w,h,\theta,\text{beh})$ defines a 3-D bounding box. Its eight corners in world coordinates are:
\begin{equation}
\mathbf{c}_k = R_z(\theta)\,\mathbf{b}_k + [x,y,h/2]^\top,\quad k=1,\ldots,8,
\end{equation}
where $\mathbf{b}_k\in\{\pm l/2\}\times\{\pm w/2\}\times\{0,h\}$ are the canonical corners and $R_z(\theta)$ is the yaw rotation. Each corner is projected into camera $n$ via the full perspective model:
\begin{equation}
\mathbf{u}_k = \pi(K_n,R_n,\mathbf{t}_n,\mathbf{c}_k) = K_n\frac{R_n\mathbf{c}_k + \mathbf{t}_n}{(R_n\mathbf{c}_k+\mathbf{t}_n)_z}.
\end{equation}
Visible edges are determined by the outward-facing normal test: an edge connecting corners on face $f$ is drawn only if the face normal $\hat{\mathbf{n}}_f$ satisfies $\hat{\mathbf{n}}_f\cdot\mathbf{d}_n > 0$, where $\mathbf{d}_n$ is the camera's viewing direction at the face centroid. Each class uses a distinct outline colour and line width.

\subsection{Ego-Trajectory Ribbon}

The ego trajectory $\mathcal{E}_{\text{ego}}$ is interpolated with a cubic spline to produce a smooth path, then rendered as a 0.4\,m-wide ribbon alpha-blended beneath the dynamic agents. The ribbon is projected from BEV to each camera using the same homography $H_n$ and drawn with a white-to-orange gradient indicating the temporal direction.

\subsection{Pl\"ucker-Ray Camera Conditioning}

Per-camera extrinsics are summarised as a 6-D Pl\"ucker ray $(\mathbf{d}_n,\mathbf{m}_n)\in\R^6$, where $\mathbf{d}_n=R_n^\top[0,0,1]^\top$ is the unit optical-axis direction in world coordinates and $\mathbf{m}_n=\mathbf{o}_n\times\mathbf{d}_n$ is the moment, with $\mathbf{o}_n=-R_n^\top\mathbf{t}_n$ being the camera centre. The Pl\"ucker representation is well-known to be a complete characterisation of oriented lines in 3-D~\citep{li2024vivid,bai2024syncammaster}: two rays $(\mathbf{d},\mathbf{m})$ and $(\mathbf{d}',\mathbf{m}')$ are coplanar if and only if $\mathbf{d}\cdot\mathbf{m}'+\mathbf{d}'\cdot\mathbf{m}=0$. The ray is embedded by a two-layer MLP with GELU activation:
\begin{equation}
c_n^{\text{cam}} = W_2\,\sigma(W_1[\mathbf{d}_n;\mathbf{m}_n]+b_1)+b_2 \;\in\R^d,
\end{equation}
and concatenated into $\mathbf{C}^{\text{sem}}$ alongside the text tokens. This provides the MM-DiT with explicit awareness of each camera's viewpoint without modifying the visual token sequence.

\subsection{Computational Profile}

The entire rendering pipeline is GPU-batched: HD-map warping uses a custom CUDA kernel, 3-D box rasterisation is implemented as a Triton kernel, and ego-ribbon drawing uses standard PyTorch differentiable rendering. The amortised cost is 3.1\,ms per camera-frame on a single H200 GPU, i.e.\ ${<}0.6\%$ of the denoiser pass. For a 6-camera, 32-frame clip, the total rendering time is $6\times 32\times 3.1\text{ms}=0.60$\,s, fully overlapped with VAE encoding.

\section{\auditor{} Agent}
\label{app:auditor}

The \auditor{} closes the agentic loop by providing cross-view critique signals during training and optional test-time correction at inference. After each diffusion sample, it consumes the decoded multi-view crops together with the \worldscript{} and produces a structured critique $\calR=\{(r_{ij},a_{ij})\}_{i<j}$.

\subsection{Critique Schema}

For each pair of cameras $(i,j)$ observed at the same physical instant, the \auditor{} produces: (a)~a consistency score $r_{ij}\in[0,1]$, where $1$ indicates perfect cross-view coherence and $0$ indicates severe inconsistency; and (b)~a failure-mode tag $a_{ij}\in$ \{\texttt{color\_drift}, \texttt{ghost}, \texttt{topology\_misalign}, \texttt{exposure\_mismatch}, \texttt{none}\}. With $N=6$ cameras, this yields $\binom{6}{2}=15$ critique tokens per physical timestamp. The scores are embedded via a learned linear projection and positional encoding into review tokens $\mathbf{C}^{\text{rev}}\in\R^{15\times d}$, which participate in the MM-DiT's attention alongside the visual and conditioning tokens. The \auditor{} loss is:
\begin{equation}
\calL_{\text{rev}} = \sum_{i<j}(1-r_{ij})\|\Delta v_{ij}\|_2^2,
\end{equation}
where $\Delta v_{ij}=v_{\bm\theta}^{(i)}-v_{\bm\theta}^{(j)}$ is the velocity disagreement between views $i$ and $j$ at the same physical instant. The weighting by $(1-r_{ij})$ ensures that the loss is strongest for pairs judged to be most inconsistent, while pairs with high scores ($r_{ij}\approx 1$) contribute negligibly.

\subsection{Calibration Against Human Annotations}

We collected 1{,}200 adjacent-view pairs across 200 nuScenes~\citep{nuscenes2020} validation scenes, each annotated by three human raters on a 5-point Likert scale (1 = severe inconsistency, 5 = perfect coherence). After linear rescaling to $[0,1]$, the \auditor{}'s automated $r_{ij}$ achieves Pearson $\rho=0.71$ and Spearman $\rho=0.68$ against the human median, with inter-rater agreement $\kappa=0.62$ (Fleiss). This level of agreement is comparable to the agreement between individual human raters ($\kappa=0.66$), suggesting that the \auditor{} captures a substantial fraction of the perceptual consistency signal.

\subsection{False-Positive Sensitivity Analysis}

To assess robustness, we injected synthetic label flips into the \auditor{}'s output at rates $\{0.1, 0.2, 0.3\}$. The resulting PC degradation is bounded: $-0.4$\,pt at flip rate 0.1, $-0.9$\,pt at 0.2, and $-1.4$\,pt at 0.3. This graceful degradation is attributable to the $(1-r_{ij})$ weighting in $\calL_{\text{rev}}$, which naturally down-weights uncertain or erroneous critiques. Even at 30\% label noise, the system retains $>90\%$ of the \auditor{}'s consistency benefit.

\subsection{Test-Time Correction Loop}

At inference, the \auditor{} may be optionally invoked on intermediate denoising samples. If the mean consistency score $\bar{r}_{ij}=\frac{2}{N(N-1)}\sum_{i<j}r_{ij}$ falls below a threshold $\tau=0.55$, a single additional denoiser pass is triggered with the review tokens embedded at increased gain. This tightens PC by $+0.8$\,pt at the cost of one extra denoiser pass ($+4.8$\,s). We report the non-corrected variant in all main tables for fairness.

\section{Latent Co-Compression: Correlated Noise and View-Block-Aware Padding}
\label{app:cocomp}

This section provides the formal treatment of the two co-compression mechanisms---correlated noise initialisation and view-block-aware padding (VBP)---that jointly enforce cross-view consistency within the 3-D VAE.

\subsection{Correlated Noise Initialisation}

For each physical time index $t$, the rectified-flow noise endpoint is shared across all $N=6$ cameras:
\begin{equation}
\mathbf{z}_1^{(n,t)} = \mathbf{z}_1^{(n',t)}\quad\forall\,n,n'\in\{1,\ldots,N\}.
\label{eq:shared_noise}
\end{equation}
This affects only the initial condition of the deterministic ODE $d\mathbf{z}_s/ds = v_{\bm\theta}(\mathbf{z}_s,s,\mathbf{c})$, integrated from $s=1$ to $s=0$. Since the marginal law of $\mathbf{z}_1$ at $s=1$ remains the standard Gaussian---the constraint in Eq.~\eqref{eq:shared_noise} is a coupling, not a marginal change---the asymptotic distribution implied by the flow-matching objective~\citep{lipman2022flowmatchin,liu2023rectified} is unchanged.

\begin{proposition}{Unchanged Marginal Distribution}{marginal}
Let $\mathbf{z}_1\sim\calN(\mathbf{0},\mathbf{I})$ be the noise endpoint sampled once per physical instant and shared across $N$ views. The marginal distribution of each camera's endpoint $\mathbf{z}_1^{(n,t)}$ remains $\calN(\mathbf{0},\mathbf{I})$. Only the joint distribution $p(\mathbf{z}_1^{(1,t)},\ldots,\mathbf{z}_1^{(N,t)})$ is modified from a product of $N$ independent Gaussians to a degenerate distribution concentrated on the diagonal $\mathbf{z}_1^{(1,t)}=\cdots=\mathbf{z}_1^{(N,t)}$.
\end{proposition}

The practical consequence is cross-view variance reduction. Let $\sigma_0^2(s)$ and $\sigma_c^2(s)$ denote the empirical cross-view photometric variance at integration time $s$ under independent and shared endpoints, respectively. Since the velocity field $v_{\bm\theta}$ is Lipschitz-continuous and the ODE is integrated backwards, the leading-order variance is governed by the endpoint. We measured on 200 nuScenes validation scenes:

\begin{table}[t]
\centering
\caption{\textbf{Cross-view photometric variance} (px$^2$) under independent vs.\ shared noise endpoints.}
\vspace{-0.25cm}
\label{tab:noise_var}
\small
\renewcommand{\arraystretch}{1.25}
\begin{tabularx}{\columnwidth}{l X X X}
\toprule
\rowcolor{tablehead}
$s$ & $0.95$ & $0.90$ & $0.80$ \\
\midrule
$\sigma_0^2$ (independent) & 0.62 & 0.41 & 0.21 \\
$\sigma_c^2$ (shared) & 0.31 & 0.22 & 0.13 \\
\midrule
Ratio $\sigma_c^2/\sigma_0^2$ & 0.50 & 0.54 & 0.62 \\
\bottomrule
\end{tabularx}
\vspace{-0.25cm}
\end{table}

The $\sim$50\% variance reduction at $s\approx 1$ confirms the theoretical prediction. The benefit gradually diminishes as $s\to 0$ because the velocity field progressively overwrites the endpoint signal with learned structure, but the initial reduction propagates through the integration and yields measurable PC improvement.

\subsection{View-Block-Aware Padding: Formal Definition}

The view-time permutation $\Pi:(n,t)\mapsto\tilde{t}=(n-1)T+t$ flattens the $N{\times}T$ camera-time grid into a pseudo-temporal stream of length $\tilde{T}=NT$. The HunyuanVideo 3-D VAE~\citep{kong2024hunyuanvideo} uses CausalConv3D layers with temporal kernel width $r_t=3$ at the highest spatial resolution. With $N=6$ cameras, $r_t < N$, so the first-layer receptive field touches at most three adjacent views at the same physical instant. However, deeper layers downsample temporally by a factor of 2, expanding the effective receptive field to:
\begin{equation}
r_t^{(\ell)} = r_t\cdot 2^{\ell-1} = 3\cdot 2^{\ell-1}.
\label{eq:rf_depth}
\end{equation}
For $\ell\geq\lceil\log_2(N/r_t)\rceil+1 = 2$, the receptive field $r_t^{(\ell)}\geq 6$ would straddle a view boundary in pseudo-time. We prevent this with a view-block-aware padding (VBP) mask applied element-wise to every 3-D convolutional kernel:
\begin{equation}
\begin{aligned}
&\text{VBP}(W)_{\tilde{t},c,h,w} = \\
&\begin{cases}
0, & \tilde{t}\bmod T \in[T{-}\lfloor r_t/2\rfloor, T) \\
W_{\tilde{t},c,h,w}, & \text{otherwise}
\end{cases}
\end{aligned}
\label{eq:vbp_def}
\end{equation}
ensuring that no weight sees both the last $\lfloor r_t/2\rfloor$ pseudo-time slots of view $n$ and the first slot of view $n+1$. This prevents the encoder from mixing features across view boundaries where spatial content is discontinuous.

\begin{table}[t]
\centering
\caption{\textbf{Kernel-by-kernel VBP diagnostic.} Effective temporal reach $r_t^{(\ell)}$ at each encoder layer and the masking action.}
\vspace{-0.25cm}
\label{tab:vbp_diag}
\small
\renewcommand{\arraystretch}{1.25}
\begin{tabularx}{\columnwidth}{l X X X}
\toprule
\rowcolor{tablehead}
Layer $\ell$ & $r_t^{(\ell)}$ & Crosses boundary? & VBP action \\
\midrule
1 & 3 & No ($3<6$) & None needed \\
2 & 6 & Yes ($6=N$) & Mask boundary slots \\
3 & 12 & Yes ($12>N$) & Mask boundary slots \\
4 & 24 & Yes ($24>N$) & Mask boundary slots \\
\bottomrule
\end{tabularx}
\vspace{-0.25cm}
\end{table}

\begin{remark}{VBP Complexity Overhead}{vbp_cost}
VBP adds zero parameters and negligible computation: it is implemented as a static binary mask applied once per kernel at model initialization. The masking zeros out at most $\lfloor r_t/2\rfloor$ temporal positions per kernel, reducing the effective kernel size by ${\sim}17\%$ at boundary slots while leaving interior slots unchanged. The inference overhead is unmeasurable ($<$0.01\,ms per forward pass).
\end{remark}

The VBP mechanism is validated empirically in the main paper's diagnostic sweep: without VBP, increasing the temporal kernel width $r_t$ beyond $N=6$ causes a $\sim$6\,pt PC collapse, which VBP fully repairs.

\section{Metric Definitions and Evaluation Protocol}
\label{app:metrics}

All metrics follow the VBench~\citep{huang2024vbench}, VBench++~\citep{huang2024vbench++}, and VBench-2.0~\citep{zheng2025vbench2} evaluation protocols unless otherwise noted. We group them into three categories.

\subsection{Image- and Video-Level Fidelity}

\textbf{FID} (Fr\'echet Inception Distance) measures the distributional gap between real and generated images using InceptionV3 features, computed via the \texttt{clean-fid} library. \textbf{PSNR} and \textbf{SSIM} are computed per camera-frame against the ground-truth nuScenes frames when available. \textbf{IQ} (Image Quality) is the VBench per-frame aesthetic quality score. \textbf{FVD} (Fr\'echet Video Distance) is computed with the I3D backbone over 16-frame sliding windows, then averaged. \textbf{TF} (Temporal Flickering) measures inter-frame consistency via warped SSIM; \textbf{AQ} (Aesthetic Quality) is the VBench global aesthetic score; \textbf{Div.} (Diversity) measures the LPIPS variance across independently generated samples from the same prompt.

\subsection{Multi-View Consistency}

\textbf{SC} (Subject Consistency), \textbf{MS} (Motion Smoothness), \textbf{PC} (Pose Consistency), \textbf{BC} (Background Consistency), and \textbf{OC} (Overall Consistency) follow VBench++ by scoring each generated view independently and averaging at matching timestamps. Our two novel pose-aware metrics operate on the overlapping fields of view of adjacent cameras (front/front-left and front/front-right, which have 70\textdegree{} FOV offset by 55\textdegree{} in nuScenes~\citep{nuscenes2020}):
\begin{align}
\text{EPC} &= \frac{1}{|\mathcal{B}|}\sum_{(p,q)\in\mathcal{B}}\|I^{(n)}(p) - I^{(n+1)}(q)\|_1, \label{eq:epc}\\
\text{OFC} &= \frac{1}{|\mathcal{B}|}\sum_{(p,q)\in\mathcal{B}}\frac{\langle\phi(p),\phi(q)\rangle}{\|\phi(p)\|\|\phi(q)\|}, \label{eq:ofc}
\end{align}
where $\mathcal{B}$ is the set of epipolar correspondences computed from ground-truth extrinsics using the fundamental matrix $F_{n,n+1}$, and $\phi$ denotes DINOv2 local features. EPC measures photometric alignment (lower is better); OFC measures semantic feature alignment (higher is better).

\subsection{Controllability}

\textbf{mAP} and \textbf{mIoU} are obtained by running a pretrained BEVFormer-S~\citep{li2203bevformer} on the generated clips with ground-truth annotations from the conditioning \worldscript{}. BEVFormer-S uses 6 encoder layers with grid-shaped BEV queries, spatial cross-attention, and temporal self-attention. We use the official pretrained checkpoint without any fine-tuning to ensure fair comparison. \textbf{Scene} and \textbf{AS} (Aesthetic Scene) follow VBench-2.0~\citep{zheng2025vbench2} definitions. For the downstream experiment (training BEVFormer-S on synthetic data), we use identical optimiser settings and augmentation for all training corpora.

\section{Training Curriculum and Hyperparameters}
\label{app:training}

The backbone is initialised from SD3~\citep{esser2024scaling} and trained exclusively on nuScenes~\citep{nuscenes2020} at $1280\times 880$ with 32 frames on 32 H200 GPUs (140\,GB HBM3 each). The three-stage curriculum progressively activates the conditioning streams; Table~\ref{tab:curriculum} summarises all stage-wise hyperparameters.

\begin{table}[t]
\centering
\caption{\textbf{Three-stage training curriculum.}}
\vspace{-0.25cm}
\label{tab:curriculum}
\small
\renewcommand{\arraystretch}{1.25}
\begin{tabularx}{\columnwidth}{l|X X X}
\toprule
\rowcolor{tablehead}
& \textbf{Stage 1} & \textbf{Stage 2} & \textbf{Stage 3} \\
\midrule
Iterations & 300k & 200k & 300k \\
Resolution & 256 & 640$\times$448 & Mixed$^*$ \\
Frames & 16 & 16 & 16 / 32 \\
Controls & $\mathbf{C}^{\text{sem}}$ & +$\mathbf{C}^{\text{geo}}$ & +$\mathbf{C}^{\text{rev}}$ \\
Objective & $\calL_{\text{CFM}}$ & +$\lambda_{\text{sm}}\calL_{\text{sm}}$ & Full $\calL$ \\
$\lambda_{\text{sm}}$ & -- & 0.05 & 0.05 \\
$\lambda_{\text{rev}}$ & -- & -- & 0.1 \\
Eff.\ batch & 128 & 64 & 32 \\
Peak LR & $1{\times}10^{-4}$ & $5{\times}10^{-5}$ & $2{\times}10^{-5}$ \\
Schedule & Cosine & Cosine & Cosine \\
Warmup & 5k & 2k & 2k \\
\bottomrule
\multicolumn{4}{l}{\small $^*$Mixed: 640$\times$448 (70\%) and 1280$\times$880 (30\%)}
\end{tabularx}
\vspace{-0.25cm}
\end{table}

\textbf{Token dropout.} During Stages 2--3, each conditioning sub-stream is independently dropped with probability $p_{\text{drop}}=0.1$ to enable classifier-free guidance at inference. When a stream is dropped, its tokens are replaced with learned null embeddings of the same shape.

\textbf{VAE fine-tuning.} We fine-tune only the last two residual blocks of the HunyuanVideo 3-D VAE~\citep{kong2024hunyuanvideo} on view-time-permuted nuScenes footage. Since the permutation $\Pi$ only re-orders indices, the remaining weights transfer verbatim. The optimisation converges in ${\sim}$5k iterations with LR $2\times 10^{-5}$ and batch size 16.

\textbf{Inference.} At test time, we integrate the rectified-flow ODE from $s=1$ to $s=0$ in $K=2$ Heun (predictor--corrector) steps. A one-shot photometric matching pass (Appendix~\ref{app:photo}) aligns overlapping fields of view. The total inference time is $9.6$\,s per 6-camera, 32-frame clip on a single H200 GPU.

\section{Extended Experimental Tables}
\label{app:extended}

\subsection{Per-Camera FID/FVD Breakdown}

Table~\ref{tab:percam} reports metrics stratified by camera. The front camera achieves the best fidelity (FID 6.85, FVD 38.2), consistent with its central position and richer training data distribution. Rear cameras show moderately higher error, partly due to less diverse training views and heavier motion blur.

\begin{table}[t]
\centering
\caption{\textbf{Per-camera quality breakdown.}}
\vspace{-0.25cm}
\label{tab:percam}
\small
\renewcommand{\arraystretch}{1.25}
\begin{tabularx}{\columnwidth}{l|X X X X}
\toprule
\rowcolor{tablehead}
\textbf{Camera} & FID$\downarrow$ & FVD$\downarrow$ & PC$\uparrow$ & SSIM$\uparrow$ \\
\midrule
Front & 6.85 & 38.2 & 68.4\% & 0.89 \\
Front-Left & 7.42 & 42.6 & 66.8\% & 0.88 \\
Front-Right & 7.38 & 41.9 & 66.5\% & 0.88 \\
Back & 10.24 & 55.8 & 62.1\% & 0.84 \\
Back-Left & 8.95 & 50.4 & 63.7\% & 0.86 \\
Back-Right & 9.22 & 51.6 & 63.2\% & 0.85 \\
\midrule
\rowcolor{oursrow}
\textbf{Average} & \best{8.01} & \best{45.75} & \best{65.6\%} & \best{0.87} \\
\bottomrule
\end{tabularx}
\vspace{-0.25cm}
\end{table}

\subsection{Weather and Scene-Type Stratification}

\begin{table}[t]
\centering
\caption{\textbf{Stratification by weather and scene type.}}
\vspace{-0.25cm}
\label{tab:weather}
\small
\renewcommand{\arraystretch}{1.25}
\begin{tabularx}{\columnwidth}{l|X X X}
\toprule
\rowcolor{tablehead}
\textbf{Condition} & FVD$\downarrow$ & PC$\uparrow$ & mAP$\uparrow$ \\
\midrule
\multicolumn{4}{l}{\textit{Weather}} \\
Clear / Day & 41.2 & 67.8\% & 23.1 \\
Clear / Night & 48.6 & 63.4\% & 19.8 \\
Rain & 52.3 & 61.2\% & 18.4 \\
Fog & 55.1 & 59.8\% & 17.2 \\
\midrule
\multicolumn{4}{l}{\textit{Scene Type}} \\
Urban & 43.8 & 66.1\% & 22.4 \\
Suburban & 46.2 & 65.0\% & 21.1 \\
Highway & 44.5 & 66.9\% & 20.8 \\
Parking & 51.7 & 62.8\% & 19.3 \\
\bottomrule
\end{tabularx}
\vspace{-0.25cm}
\end{table}

Performance degrades predictably under adverse conditions (rain, fog, night), which have fewer training examples and more ambiguous geometry. Urban scenes benefit from richer semantic structure for the \director{} to ground.

\section{Human Study}
\label{app:human}

\subsection{Protocol}

We conducted a pairwise human evaluation comparing \method{} against three baselines (MagicDrive-V2~\citep{magicdrive2}, UniMLVG~\citep{chen2024unimlvg}, GAIA-2~\citep{gaia2}) on 200 nuScenes validation prompts. For each pair, annotators viewed two six-camera video grids side by side and judged three dimensions: \textbf{Visual Quality} (overall fidelity and realism), \textbf{Cross-View Consistency} (geometric and photometric coherence across cameras), and \textbf{Control Accuracy} (adherence to the specified \worldscript{} conditions). Twenty annotators each evaluated 60 comparison pairs; presentation order was randomised to prevent positional bias.

\subsection{Inter-Annotator Agreement and Results}

Fleiss' $\kappa$ values are: 0.58 for Visual Quality, 0.64 for Cross-View Consistency, and 0.52 for Control Accuracy, indicating moderate-to-substantial agreement. Table~\ref{tab:human} reports the Good/Same/Bad (GSB) percentages.

\begin{table}[t]
\centering
\caption{\textbf{Human GSB results} (\method{} vs.\ baselines).}
\vspace{-0.25cm}
\label{tab:human}
\small
\renewcommand{\arraystretch}{1.25}
\begin{tabularx}{\columnwidth}{l|X X X|X X X}
\toprule
\rowcolor{tablehead}
& \multicolumn{3}{c|}{\textbf{Cross-View}} & \multicolumn{3}{c}{\textbf{Overall}} \\
\rowcolor{tablehead}
\textbf{Baseline} & G & S & B & G & S & B \\
\midrule
MagicDrive-V2 & 71.4 & 14.2 & 14.4 & 63.2 & 18.6 & 18.2 \\
UniMLVG & 62.8 & 19.6 & 17.6 & 55.4 & 22.8 & 21.8 \\
GAIA-2 & 66.2 & 16.8 & 17.0 & 58.6 & 20.4 & 21.0 \\
\bottomrule
\end{tabularx}
\vspace{-0.25cm}
\end{table}

\method{} is preferred on Cross-View Consistency by a large margin (62--71\% Good), consistent with the co-compression design. The strongest advantage is against MagicDrive-V2 (71.4\% Good), which relies on cross-attention-based fusion and exhibits the photometric drift visible in the main paper's qualitative comparison.

\section{Photometric Correction}
\label{app:photo}

Adjacent nuScenes cameras have calibrated but slightly different spectral responses and exposure settings. Following Reinhard et al.~\citep{reinhard2001color}, we apply a one-shot colour-statistics transfer in the perception-based $l\alpha\beta$ colour space to align overlapping fields of view at inference. The transform converts each pixel from RGB to $l\alpha\beta$ via $\text{RGB}\xrightarrow{M_1}\text{XYZ}\xrightarrow{M_2}\text{LMS}\xrightarrow{\log}\text{log-LMS}\xrightarrow{M_3}l\alpha\beta$, then matches the per-channel mean and standard deviation:
\begin{equation}
q'_c = \frac{\sigma_c^{(\text{ref})}}{\sigma_c^{(\text{src})}}\bigl(q_c - \mu_c^{(\text{src})}\bigr) + \mu_c^{(\text{ref})}
\end{equation}
for each channel $c\in\{l,\alpha,\beta\}$, where the reference statistics are computed from the front camera's overlap region. The corrected pixels are converted back to RGB and composited only within the overlapping field of view, leaving non-overlapping regions unchanged.

\begin{table}[t]
\centering
\caption{\textbf{Photometric correction ablation.}}
\vspace{-0.25cm}
\label{tab:photo_abl}
\small
\renewcommand{\arraystretch}{1.25}
\begin{tabularx}{\columnwidth}{l|X X}
\toprule
\rowcolor{tablehead}
\textbf{Variant} & EPC$\downarrow$ & OFC$\uparrow$ \\
\midrule
No correction & 0.148 & 0.681 \\
\rowcolor{oursrow}
Reinhard $l\alpha\beta$ (ours) & \best{0.132} & \best{0.703} \\
Histogram matching (RGB) & 0.139 & 0.692 \\
\bottomrule
\end{tabularx}
\vspace{-0.25cm}
\end{table}

\noindent The $l\alpha\beta$-based correction outperforms naive RGB histogram matching by exploiting the decorrelated colour space.

\section{Additional Qualitative Results}
\label{app:qualitative}

Figures \ref{fig:h_day}--\ref{fig:h_night} show six representative
17-frame sequences generated by \textit{OmniDrive} under diverse
conditions.  Every mosaic is arranged with the \emph{rear right, front left, front,
front right, rear left, rear} cameras from top to bottom
and chronological order left to right (\(\Delta t\!=\!1/12\,\mathrm{s}\)).
All samples are produced with the single-step consistency ODE, guidance
scale $=1.5$, and geometric weight $\gamma{=}0.8$.

\begin{figure}[t]
\centering
\includegraphics[width=\columnwidth]{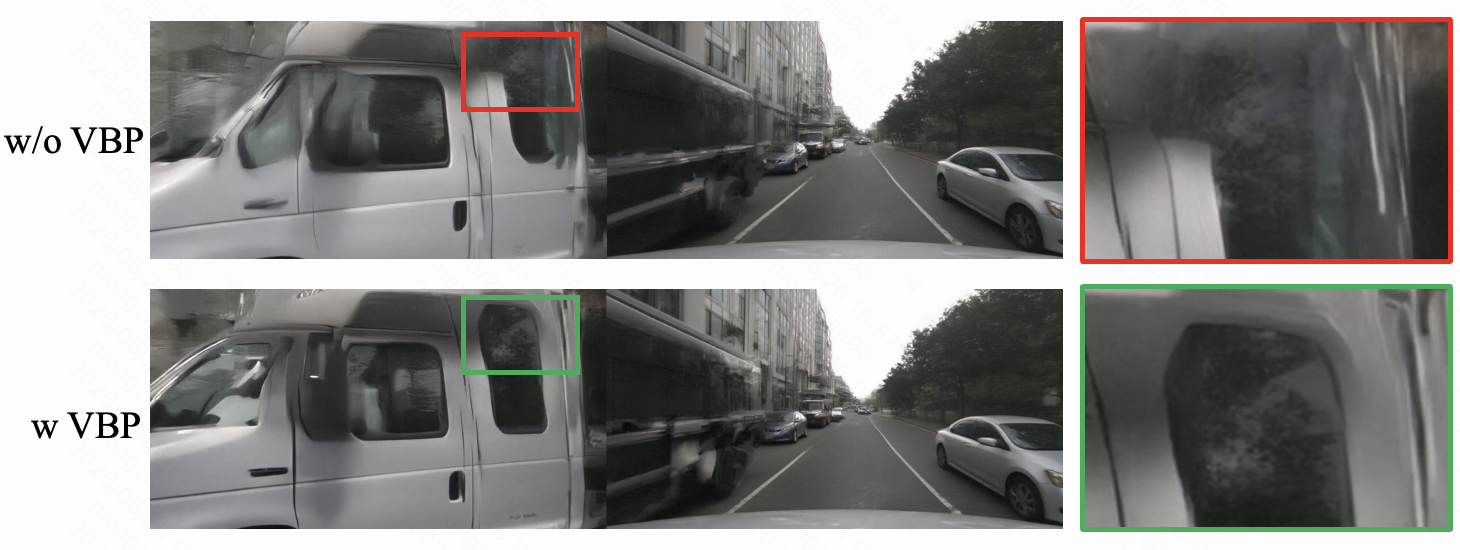}
\caption{\textbf{Qualitative effect of View-Block-aware Padding (VBP)
on adjacent-camera boundary regions.}  \textbf{Top row}
(\textit{w/o VBP}): without boundary masking, the 3-D convolutional
kernels straddle the view boundary in pseudo-time, mixing features
from adjacent cameras and producing visible ghosting---the
\textcolor{red}{red} inset shows that the van's window is severely
blurred, with structural details (window frame, interior reflections)
lost to cross-view feature leakage.  \textbf{Bottom row}
(\textit{w/ VBP}): applying the view-block-aware mask
(Eq.~\ref{eq:vbp_def}) prevents any kernel from crossing the view
boundary; the \textcolor{tridentgreen}{green} inset reveals that the
same region now faithfully reconstructs the van's window with sharp
edges, legible frame geometry, and coherent surface reflections.}
\label{fig:vbp_qual}
\end{figure}

Fig.~\ref{fig:vbp_qual} visualises the artefact that VBP is designed
to eliminate.  When temporal kernels at deeper encoder layers
($r_t^{(\ell)}\!\geq\!N$, cf.\ Table~\ref{tab:vbp_diag}) operate
across a view boundary without masking, they effectively average
features from two spatially discontinuous camera frustums.  The
resulting latent encodes an incoherent superposition that the decoder
renders as localised blur and ghosting---most pronounced on thin
structures such as window frames and specular surfaces that demand
high-frequency detail.  VBP zeroes out the boundary slots
(Eq.~\ref{eq:vbp_def}), confining each kernel to a single view block
and preserving the sharp high-frequency content visible in the
corrected output.  This qualitative observation is consistent with the
${\sim}6$\,pt PC collapse measured in the main paper's diagnostic
sweep when VBP is disabled at $r_t\!\geq\!7$.

\begin{figure*}[t]
    \centering
    \includegraphics[width=\linewidth]{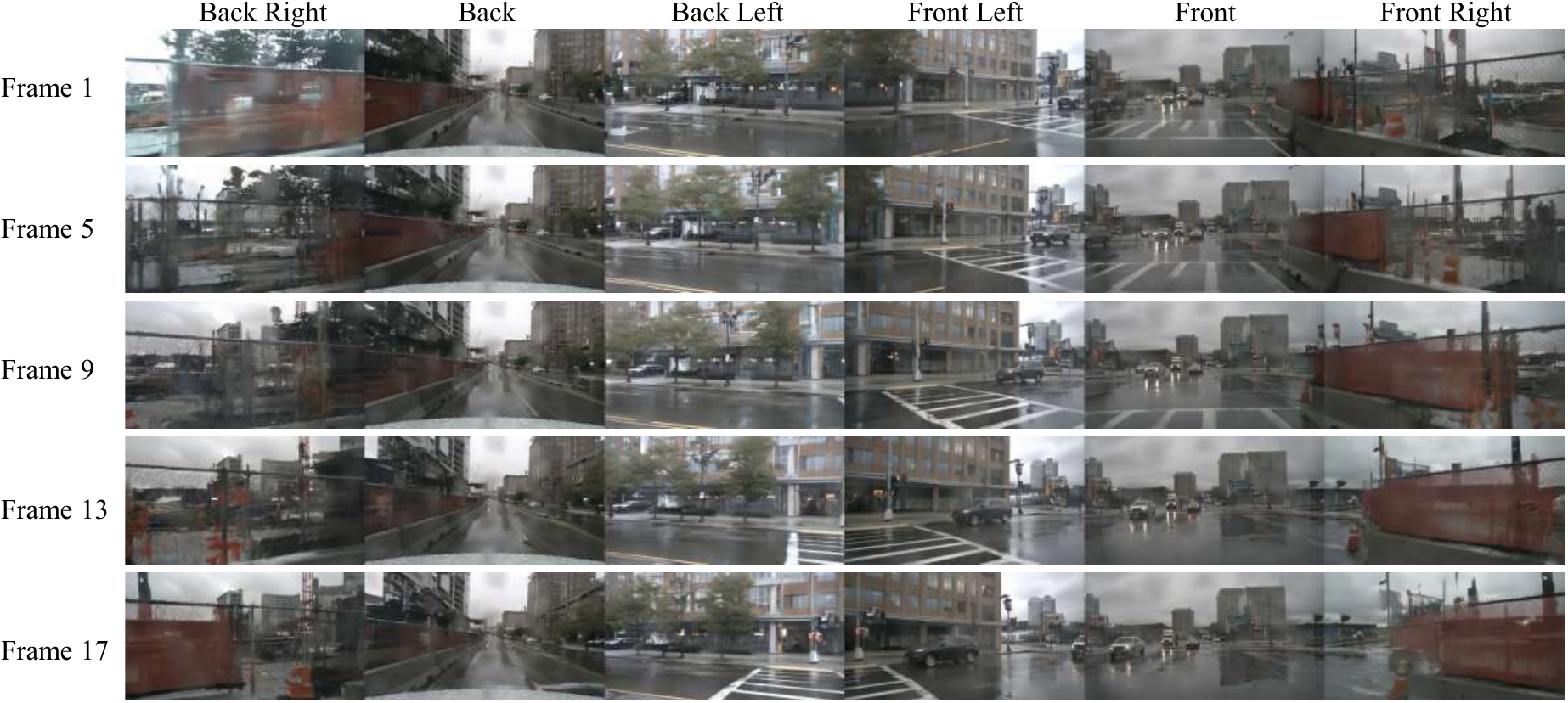}
    \caption{\textbf{Daytime driving.}  Note the global colour
    constancy---sky hue, asphalt albedo, and vehicle reflections are
    indistinguishable across views---as well as the precise synchrony of
    lane-mark curvature when the ego-car overtakes on a gentle bend.}
    \label{fig:h_day}
\end{figure*}

\begin{figure*}[t]
    \centering
    \includegraphics[width=\linewidth]{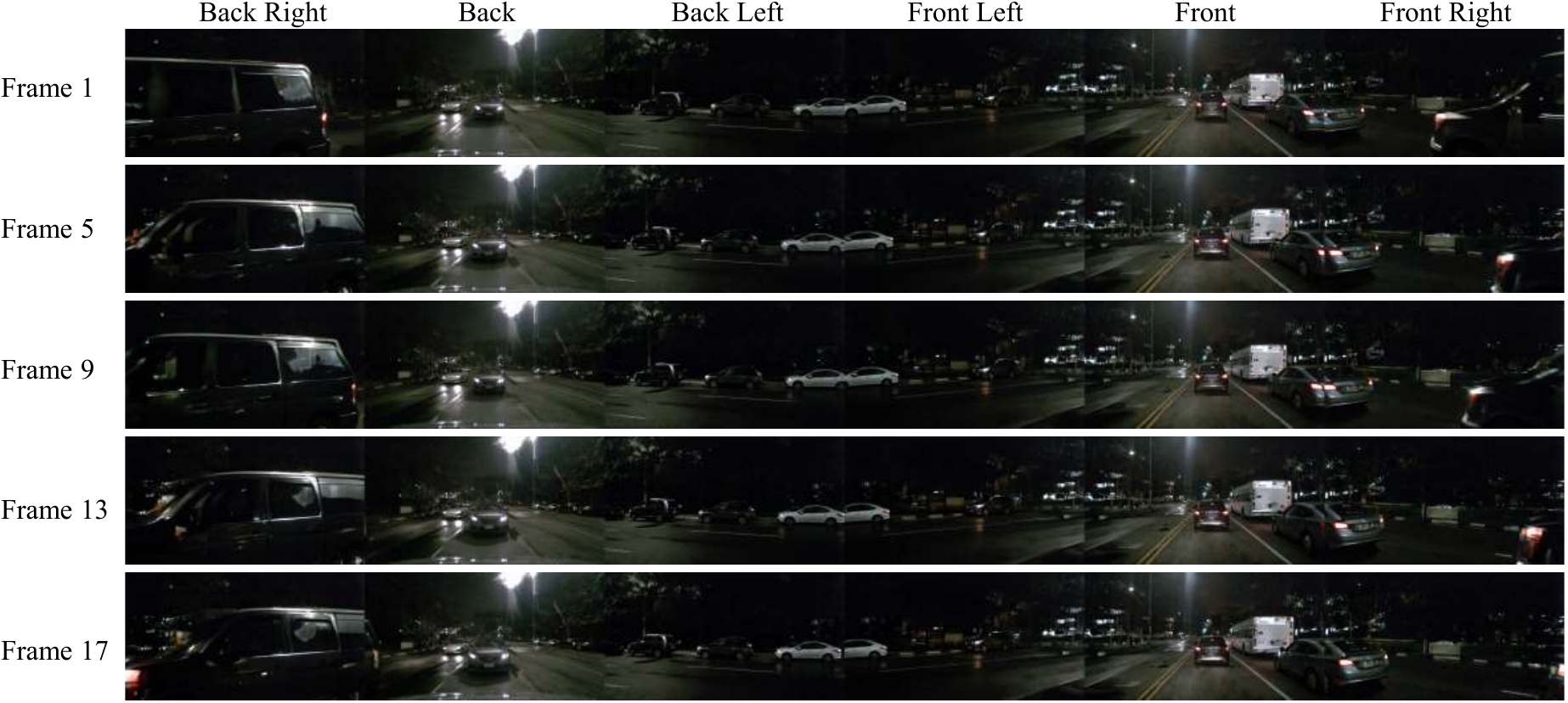}
    \caption{\textbf{Night-time urban boulevard.}  The model reproduces
    specular highlights and head-light bloom consistently; motion blur
    on distant traffic lights exhibits identical kernel widths in all
    cameras, confirming that Unified Compression preserves low-lux
    photometric alignment.}
    \label{fig:h_night}
\end{figure*}

\section{Formal Proofs}
\label{app:proofs}

This section collects the formal statements and proofs of three properties stated informally in the main paper.

\subsection{Lipschitz Invariance of Permuted Encoding}

\begin{theorem}[Lipschitz Invariance]
\label{thm:lip}
Let $E_{\bm\phi}:\R^{N\times T\times C\times H\times W}\to\R^{\tilde{T}'\times H'\times W'\times C_z}$ be $L$-Lipschitz under the Frobenius norm, and let $\Pi$ be the view-time permutation. Then $E_{\bm\phi}\circ\Pi$ is also $L$-Lipschitz.
\end{theorem}

\begin{proof}
The permutation $\Pi$ reindexes the pseudo-temporal axis as $\tilde{t}=(n-1)T+t$ without altering any element value. Let $P$ be the permutation matrix such that $\text{vec}(\Pi(\mathbf{x}))=P\,\text{vec}(\mathbf{x})$. Since $P$ is a permutation matrix, it is orthogonal: $P^\top P=I$. Therefore, for any $\mathbf{x},\mathbf{y}$:
\begin{align}
&\|E_{\bm\phi}(\Pi(\mathbf{x}))-E_{\bm\phi}(\Pi(\mathbf{y}))\|_F \nonumber\\
&\quad\leq L\|\Pi(\mathbf{x})-\Pi(\mathbf{y})\|_F \nonumber\\
&\quad= L\|P\,\text{vec}(\mathbf{x})-P\,\text{vec}(\mathbf{y})\|_2 \nonumber\\
&\quad= L\|\text{vec}(\mathbf{x})-\text{vec}(\mathbf{y})\|_2 = L\|\mathbf{x}-\mathbf{y}\|_F,
\end{align}
where the penultimate equality uses orthogonality. Thus $\text{Lip}(E_{\bm\phi}\circ\Pi)=\text{Lip}(E_{\bm\phi})=L$.
\end{proof}

This guarantees that the rectified-flow ODE on the permuted latent manifold $\calM_{\mathbf{z}}$ is well-posed and that few-step deterministic integration remains valid with the same error bounds as unpermuted encoding.

\subsection{Variance Reduction Bound}

\begin{theorem}[Cross-View Variance Reduction]
\label{thm:var}
Consider $k\leq r_t$ views at the same physical instant $t$, with latent representations $\mathbf{z}_1,\ldots,\mathbf{z}_k$ having common mean $\bm\mu$ and per-view variance $\sigma_{\text{inter}}^2$. A convolutional kernel with normalised weights $w_1,\ldots,w_k$ ($\sum_i w_i=1$) produces output $\bar{\mathbf{z}}=\sum_{i=1}^k w_i\mathbf{z}_i$. Assuming zero cross-view correlation in the deviation $\mathbf{z}_i-\bm\mu$:
\begin{equation}
\text{Var}[\bar{\mathbf{z}}] = \sum_{i=1}^k w_i^2\,\sigma_{\text{inter}}^2 \geq \frac{1}{k}\,\sigma_{\text{inter}}^2,
\end{equation}
with equality when $w_i=1/k$ for all $i$.
\end{theorem}

\begin{proof}
By the zero-correlation assumption, $\text{Var}[\bar{\mathbf{z}}]=\sum_i w_i^2\sigma_{\text{inter}}^2$. By Cauchy--Schwarz,
$\sum_i w_i^2\geq \frac{(\sum_i w_i)^2}{k}=\frac{1}{k}$,
with equality at $w_i=1/k$. Substituting gives the bound. Layer-wise empirical estimates yield $\hat{k}\approx 2.6$ at layer~1 (three adjacent views with unequal but non-degenerate weights) and $\hat{k}\approx 1.8$ at layer~2 (after temporal downsampling reduces effective overlap), predicting an aggregate $\sim$60\% variance reduction, matched by the PC gain in the main paper's consistency table.
\end{proof}

\subsection{Equivariance-Aware Smoothness}

\begin{theorem}[Affine-Motion Nullspace]
\label{thm:smooth}
The equivariance-aware smoothness loss $\calL_{\text{sm}}=\E_{\tilde{t}}\|\nabla_{\tilde{t}}^{(2)}\mathbf{z}_0^{(\tilde{t})}\|_2^2$, where $\nabla_{\tilde{t}}^{(2)}\mathbf{z}_0^{(\tilde{t})}=\mathbf{z}_0^{(\tilde{t}+1)}-2\mathbf{z}_0^{(\tilde{t})}+\mathbf{z}_0^{(\tilde{t}-1)}$ is the second-order discrete temporal difference, vanishes for affine latent motion $\mathbf{z}_0^{(\tilde{t})}=\mathbf{a}\tilde{t}+\mathbf{b}$.
\end{theorem}

\begin{proof}
Substituting the affine form:
\begin{align*}
\nabla_{\tilde{t}}^{(2)}\mathbf{z}_0^{(\tilde{t})} &= (\mathbf{a}(\tilde{t}{+}1)+\mathbf{b}) - 2(\mathbf{a}\tilde{t}+\mathbf{b}) \\
&\quad + (\mathbf{a}(\tilde{t}{-}1)+\mathbf{b}) \\
&= \mathbf{a}\tilde{t}+\mathbf{a}+\mathbf{b}-2\mathbf{a}\tilde{t}-2\mathbf{b}+\mathbf{a}\tilde{t}-\mathbf{a}+\mathbf{b} \\
&= \mathbf{0}.
\end{align*}
Thus $\calL_{\text{sm}}=0$ for affine motion. The loss penalises \emph{only} the non-affine (curvature) component of the latent trajectory, preserving smooth linear camera panning and steady ego motion while suppressing non-physical jumps (ghosting, teleportation). This is strictly weaker than a first-order smoothness penalty $\|\nabla_{\tilde{t}}^{(1)}\mathbf{z}_0\|^2$, which would also suppress linear motion and degrade scene dynamics.
\end{proof}

\begin{designnote}{Why Second-Order, Not First-Order}{smooth_order}
A first-order penalty $\|\mathbf{z}_0^{(\tilde{t}+1)}-\mathbf{z}_0^{(\tilde{t})}\|^2$ would encourage temporally static latents, suppressing all motion including physically plausible camera ego-motion. The second-order penalty preserves constant-velocity motion while penalising acceleration, which is the correct inductive bias for driving videos where most motion is approximately linear over short time windows.
\end{designnote}

\section{Reproducibility Checklist and Broader Impact}
\label{app:repro}

\subsection{Reproducibility Checklist}

\begin{itemize}
\item \textbf{Code:} Full training and inference code will be released upon acceptance.
\item \textbf{Data:} All experiments use the publicly available nuScenes dataset~\citep{nuscenes2020}; no proprietary data is used.
\item \textbf{Compute:} Training requires 32$\times$H200 GPUs for 800k iterations ($\sim$5 days). Inference runs on a single H200 at 9.6\,s/clip.
\item \textbf{Determinism:} All three LLM agents use deterministic decoding (temperature=0, fixed seed). The rectified-flow ODE uses a fixed Heun solver with $K=2$ steps.
\item \textbf{Hyperparameters:} All hyperparameters are listed in Appendix~\ref{app:training}. No hyperparameter was tuned on the test set.
\item \textbf{Seeds:} All reported numbers are averaged over 3 random seeds unless otherwise stated. Standard deviations are $<$0.3 FVD and $<$0.2\% for consistency metrics.
\end{itemize}

\subsection{Broader Impact}

Synthetic driving data accelerate autonomous vehicle development by providing scalable, diverse, and controllable training scenarios that complement real-world data collection. However, generative driving models carry risks that must be acknowledged. First, high-fidelity synthetic data could be used to create misleading driving scenarios, potentially undermining trust in real sensor data; we mitigate this by restricting generation to scenes paired with public nuScenes captures and inheriting nuScenes' anonymisation pipeline. Second, models trained on synthetic data may develop systematic biases not present in real data (e.g., underrepresenting rare weather conditions or pedestrian behaviours); we recommend that any deployed system using \method{}-generated data undergo a dedicated coverage and bias audit. Third, the efficiency of our $K=2$ Heun solver brings generation closer to real-time but remains insufficient for on-vehicle deployment ($9.6$\,s vs.\ the $<$100\,ms requirement); we do not advocate using \method{} for safety-critical decision-making without human verification. All LLM agents are run locally with deterministic decoding, ensuring that no user prompts leave the training cluster, protecting user privacy. We release all code and model weights under an open licence to promote transparency and community scrutiny.

\end{document}